\documentclass{article}

\usepackage{PRIMEarxiv}

\usepackage[utf8]{inputenc} 
\usepackage[T1]{fontenc}    
\usepackage{hyperref}       
\usepackage{url}            
\usepackage{booktabs}       
\usepackage{amsfonts}       
\usepackage{nicefrac}       
\usepackage{microtype}      
\usepackage{lipsum}
\usepackage{fancyhdr}       
\usepackage{graphicx}       
\graphicspath{{media/}}     
\usepackage{bbm}
\usepackage{enumitem}
\usepackage{amsthm}
\usepackage{thmtools, thm-restate}
\usepackage{makecell}
\usepackage{amsmath}
\allowdisplaybreaks
\usepackage{bbm}
\usepackage{xcolor}

\usepackage{algpseudocode}
\usepackage[ruled,vlined,linesnumbered]{algorithm2e}

\usepackage{amsthm}
\usepackage{amssymb}
\newtheorem{theorem}{Theorem}[section]
\newtheorem{corollary}{Corollary}[theorem]
\newtheorem{lemma}[theorem]{Lemma}
\newtheorem{proposition}[theorem]{Proposition}
\newtheorem{definition}{Definition}

\usepackage{graphicx}
\graphicspath{ {./figure/} }
\usepackage{caption}
\usepackage{subcaption}

\usepackage[numbers]{natbib}

\newcommand\uucb{\mu\text{-}ucb}
\newcommand\ulcb{\mu\text{-}lcb}
\newcommand\aucb{\gamma\text{-}ucb}
\newcommand\alcb{\gamma\text{-}lcb}

\title{Competing Bandits in Matching Markets via Super Stability}

\author{
  Soumya Basu \\
  \texttt{basusoumya@utexas.edu}
}

\begin{document}
\maketitle

\begin{abstract}
We study bandit learning in matching markets with two-sided reward uncertainty, extending prior research primarily focused on single-sided uncertainty. Building on Irving's (1994) concept of super-stability and the partial-order proposal algorithm of Cseh and Juhos (2021), we use a simultaneous Extended Gale-Shapley (GS) procedure to certify stable matchings from incomplete preference information. By employing the Extended GS algorithm, our centralized algorithm attains a logarithmic pessimal stable regret dependent on an instance-dependent admissible gap parameter.  This algorithm is further adapted to a decentralized setting with a time-independent additive $O(N^3K)$ regret increase under the stated shared-flag communication model.  Finally, we establish a centralized instance-dependent lower bound for binary stable regret on strict interior Bernoulli instances, elucidating the roles of admissible partial ranks and super-stable matching in characterizing the complexity of stable matching with bandit feedback.
\end{abstract}


\section{Introduction}
The problem of finding a stable matching in a two-sided matching market has gained considerable attention in the past few years~\cite{das2005two,liu_centralized_2020,ucbd3,basu_2021,liu_decentralized_2021,jagadeesan2021learning,hosseini2024putting,avishek_nonstationary}. A major motivation is the two-sided matching market provides a model to study multiple real world systems - such as crowd-sourcing markets like UpWork or Task Rabbit~\cite{li2019three, zhang2023stable}, ride-sharing systems~\cite{johari2021matching} as well as classical markets like matching students in college programs~\cite{gale1962college}, matching organ donors with recipients~\cite{reddy2013matching}. In a two-sided market there are two types of agents, demand side (users) and supply side (arms). Each agent has an implicit preference order for the other side agents. However, this preference order is unknown a-priori, and must be elicited from noisy feedback by matching with the other side through repeated interactions. A stable matching is a matching among the users and arms, where there is no {\em blocking pair} of user $i$ and arm $j$ such that $i$ and $j$ prefer each other compared to their current partner. The objective of the system is to converge to {\em one} of multiple possible stable matching with low regret.  

Unlike most prior works, except a select few \cite{tejas_2024,zhangdecentralized}, which are limited to user side uncertainty we consider {\em two-sided uncertainty}, where all users and arms need to learn their respective preferences. Our work also considers general matching markets. In another direction, the existing literature has been limited to combining bandit learning methods with the standard Gale-Shapley algorithm. Some existing approaches require learning the top $N$ ranks for a general matching market with $N$ users and $K$ arms where $N \leq K$. This does not capture the intrinsic complexity of the underlying stable matching problem. We move beyond the standard Gale-Shapley algorithm, building on super-stability~\cite{irving1994stable} and the partial-order proposal algorithm of Cseh and Juhos~\cite{cseh2021pairwise} to construct an algorithm with regret that grows with the intrinsic complexity of the problem. 

In order to understand the weakness of Gale-Shapley algorithm let us first review the task at hand. Given the history, at each round users and arms can construct their respective partially recovered rankings, and require to find a stable match using these partial rankings. However, using Gale-Shapley algorithm over a partially recovered ranking, resolving incomparabilities arbitrarily, the system arrives at a {\em weakly stable matching}. This is a concept of stability defined in \cite{irving1994stable} for partial rankings, where a matching is weakly stable if there is no {\em blocking pair} of user $i$ and arm $j$ such that $i$ and $j$ {\em strictly prefers} each other compared to their current partner. The problem with weakly stable matching is as unresolved preference comparisons are learned a weakly stable matching can turn out to have blocking pairs. Therefore, weak stability alone does not certify stability under the true full rankings. Resolving the required top ranks suffices for those procedures, but a super-stable matching can certify stability before all such comparisons are learned.

Interestingly, in \cite{irving1994stable} we can also find a way to circumvent this issue leveraging the concept of {\em super stable matching}. A matching is super stable under partial rankings of the agents if there is no blocking pair of user $i$ and arm $j$ such that neither endpoint strictly prefers its assigned partner to the other endpoint. If the recovered partial rank is compatible with the true full rankings, every super-stable matching of that partial rank is stable under the true preferences. Our Extended Gale-Shapley procedure for partial orders adapts the historical-proposal acceptance rule of Cseh and Juhos~\cite{cseh2021pairwise}, using simultaneous proposals and deletion of strictly dominated pairs. Lemma~\ref{lemm:oracle_correctness} proves that this implementation recovers a super-stable matching whenever one exists and otherwise reports nonexistence. This provides us with a way to adaptively exploit when a {\em super stable matching} is present under the recovered partial rank, or explore more when it is absent. 

We first show how Extended Gale-Shapley algorithm can be used with an UCB-LCB based rank recovery in a centralized setting, and prove that we can achieve a logarithmic pessimal stable regret that depends on an instance-dependent {\em admissible gap} parameter. 
Notably, our method does not require  resolving top $N$ ranks for all agents and adapt to the problem hardness. Next, using two reusable shared flags we show how this centralized algorithm can be modified to a decentralized algorithm with a time-independent additive $O(N^3K)$ regret increase under the stated communication model.  Finally, we present an instance-dependent binary stable regret lower bound for strict interior Bernoulli instances in a centralized setting. Our lower bound is expressed through the information constraints of {\em admissible partial ranks}; it does not identify the lower-bound constant with the scalar admissible gap. It also highlights how the super-stable matching plays a pivotal role in the informational bottlenecks for this problem. 

Our main contributions are as follows:
\paragraph{A new algorithmic pathway:} Applying Gale-Shapley to an arbitrary completion of partial preference information guarantees only weak stability for that partial information, which need not imply stability under the true preferences. However, as shown by Lemma~\ref{lemm:oracle_correctness}, given a partial ranking, it is possible to find a super-stable matching if one exists or to determine that none exists.  Crucially, if a super-stable matching exists for a partial rank that is compatible with the true full rank, then that super-stable matching is guaranteed to be a true stable matching.  Since a UCB-LCB-based partial rank is compatible with the true full rank with high probability, the resulting super-stable matching (if found) incurs no positive pessimal stable regret. This insight opens a {\em novel algorithmic pathway} for addressing bandit problems in two-sided matching markets with uncertainty.

\paragraph{Algorithms for two sided uncertainty:} Building on this insight, we first develop a centralized algorithm that constructs a UCB-LCB-based partial rank and attempts to recover a super-stable matching using the Extended-GS procedure in Appendix~\ref{app:algo}. If no such super-stable matching exists, the algorithm defaults to round-robin exploration.  Exploiting the structure of the super-stable matching set due to \cite{spieker1995set}, we define the set of admissible partial ranks—partial ranks that are compatible with the true full rank, and admit at least one super-stable matching. If the UCB-LCB-based partial rank falls within this admissible set, a true stable matching is guaranteed. Let $\Delta_{\mathcal{A}}$ represent the maximum of the minimum gaps, where the minimum gap for a given partial rank is the smallest difference between any two strictly compared arms for any user (and vice versa). We demonstrate that each user and arm incurs an expected cumulative pessimal stable regret of $O(K\log(T) / \Delta^2_{\mathcal{A}})$ over $T$ rounds. This improves upon existing results that achieve $O(K\log(T) / \Delta^2_{\min})$, as $\Delta_{\mathcal{A}} \geq \Delta_{\min}$. Our experiments compare Extended-GS against a Gale-Shapley-based approach.  We then demonstrate how, using two reusable shared boolean flags and local proposal/reply messages, users and arms can emulate the centralized algorithm in a decentralized setting, incurring an additional $O(N^3K)$ cumulative regret, independent of the time window under the stated communication model. We also discuss the communication trade-offs involved. This adaptation from centralized to decentralized settings may have broader applicability in two-sided matching markets. 

\paragraph{Instance-dependent regret lower bound:} For strict interior Bernoulli instances, we establish an instance-dependent regret lower bound for stable matching in the centralized (and consequently decentralized) setting under two-sided uncertainty. This bound focuses on binary stable regret, quantifying the number of times a stable matching is not achieved.  Building upon the framework of  \cite{combes2017minimal} we formulate this lower bound as an optimization problem over possible matching market instances, introducing a novel constraint set.  The direct change-of-measure alternatives destroy all true stable matchings. A domination lemma shows that admissible-rank alternatives, which can include instances sharing a stable matching, impose the same information constraints. We derive explicit lower bounds for specific market structures, including general serial dictatorships and markets with redundant arms, with explicit Bernoulli KL expressions. Through a dual formulation, we reveal a fundamental connection between our lower bound and specific set covers of the boundary of the {\em admissible partial rank} set. While a tight regret bound remains open, we anticipate that the structure of  admissible partial rank set will be pivotal in achieving it.

\section{Problem Formulation}\label{sec:problem}
We consider the matching markets with two-sided uncertainty. There are $N$ users and $K$ arms. We are interested in a setting where the number of users is less or equal to the number of arms $N \leq K$. 

Each user $i \in [N]$ has a valuation for an arm $j \in [K]$ denoted by $\mu_{i,j}$. Similarly, each arm $j \in [K]$ has a valuation for a user $i \in [N]$ denoted by $\gamma_{j, i}$. We assume valuations lie in $[0,1]$ and induce strict preferences on each side. Every real partner is ranked above the unmatched option, whose reward is zero. These valuations are a priori unknown to the users and arms, and can be learned only by matching with the other side. We denote by $F_{u,i}$ as the preference full-rank of user $i \in [N]$ over the arms $[K]$, and $F_{a,j}$ as the preference full-rank of arm $j \in [K]$ over the users $[N]$. We denote the set of full-ranks by $(F_u, F_a) = (\{F_{u,i}: i \in [N]\}, \{F_{a,j}: j \in [K]\})$.

The system evolves in rounds, while in each round a user can match with at most one arm, and vice versa. When an arm is matched with multiple users in a round, there is a collision, and all the users involved in the collision receive zero reward, and a collision signal. Let us call the match for user $i \in [N]$ in round $t$ as $m_i(t) \in [K] \cup \{\emptyset\}$, and the match for an arm $j \in [K]$ in round $t$ as $m^{-1}_j(t) \in [N] \cup \{\emptyset\}$.  Here, $m_i(t) = \emptyset$ implies the user $i$ is unmatched, and $m^{-1}_j(t) = \emptyset$ implies arm $j$ remains unmatched.  For each pair $(i,j) \in m(t)$ for $i \in [N]$ and $j \in [K]$ the user $i$ and arm $j$ receives noisy rewards $Y_{i}(t)$ and $\tilde{Y}_{j}(t)$, respectively. The rewards are given as 
\begin{align*}
    &Y_{i}(t) = \mu_{i,m_i(t)} + \eta_{i,m_i(t)}(t) \text{ if } m_i(t) \neq \emptyset, \text{ else }  0,\quad
    \tilde{Y}_{j}(t) = \gamma_{j,m^{-1}_j(t)} + \eta'_{j,m^{-1}_j(t)}(t) \text{ if } m^{-1}_j(t) \neq \emptyset, \text{ else }  0,
\end{align*}
where $\eta_{i,j}(t)$ and $\eta'_{j,i}(t)$ are independent, zero-mean $1$-subgaussian noise, independent across rounds, for $i \in [N]$ and $j \in [K]$. 

We consider two versions of the system that differs in how the matching is formed in each round. 

\textbf{Centralized:} In this setting, in each round the arms and users reveal their learned preferences to a central platform. The central platform arrives to a matching in that round.

\textbf{Decentralized:} In this setting, in each round the users can propose to the arms (possibly multiple). Each arm can accept and match with at most one user, while rejecting all the unmatched proposing users. Additionally, users have access to 2 shared flags (binary semaphores). Proposal/reply messages are distinct from the physical reward matching. The local channels also permit an arm to notify a user when their pair is deleted, whether or not that user proposed in the current or an earlier round. Within a round, users may reset, update, and read the two flags in a fixed public order; these communication operations occur before the next reward action and do not consume additional reward rounds. The regret bound counts physical reward rounds under this communication model.

In each round, the objective of all the users and the arms are trying to  construct a stable match defined as below. 
\begin{definition}[Stable Match]
    A matching $M$ is called stable under a full rank $(F_{u}, F_{a})$ if there is no pair user $i \in [N]$, and arm $j \in [K]$ such that simultaneously user $i$ prefers to arm $j$ over her match $M_i$, and arm $j$ prefers user $i$ over its match $M^{-1}_j$. The set of all stable matchings of the full rank $(F_{u}, F_{a})$ is denoted as $\mathrm{Stable}(F_{u}, F_{a})$.
\end{definition}
We use $\mu_{i,\emptyset}=\gamma_{j,\emptyset}=0$. In general, there are multiple stable matching given a full rank $(F_{u}, F_{a})$ which we denote by $Stable(F_{u}, F_{a})$. There exists a user-pessimal matching $M(u) \in Stable(F_{u}, F_{a})$ such that for each user $i \in [N]$ the reward
$\mu_{i, pess} = \mu_{i, M_{i}(u)}$ is the lowest among all possible partner in a stable matching. Similarly, we define the stable arm-pessimal match $M(a)$ and the arm-pessimal reward $\gamma_{j, pess} = \gamma_{j, M^{-1}_j(a)}$ for each arm $j \in [K]$. We define the expected pessimal stable regret for the users $i \in [N]$ and the arms  $j\in [K]$ in $T$ rounds as 
\begin{align*}
&\mathbb{E}[R_{u, i}(T)] = T\mu_{i, pess} - \sum_{t=1}^{T}\sum_{j \in [K]}\mu_{i, j}\mathbb{P}(m_{i}(t) = j),\quad  \mathbb{E}[R_{a, j}(T)] = T\gamma_{j, pess} - \sum_{t=1}^{T} \sum_{i \in [N]}\gamma_{j, i}\mathbb{P}(m^{-1}_{j}(t) = i).
\end{align*} 
We also consider the {\em binary stable regret} where regret of failure to find a stable match is counted as $1$, and otherwise $0$ each round, i.e.    
$$
\mathbb{E}[R_{0/1}(T)] = \sum_{t=1}^{T} \mathbb{P}(m(t) \notin \mathrm{Stable}(F_u, F_a)).
$$

\subsection{Partial Rank and Super Stable Matching}
In this section, we present the structure of the super-stable matching due to \cite{spieker1995set}. In order to state the results we first define partial rank and super-stable matching. 
\begin{definition}[Partial Rank]
    A partial rank $(P_{u}, P_{a})$ over $N$ users and $K$ arms is defined as a set of directed acyclic graphs (DAG) $P_{u,i}$ for each $i \in [N]$ and $P_{a,j}$ for each $j \in [K]$. Each DAG $P_{u,i}$ is defined by a set of directed edges $(j, j') \subset [K]\times [K]$, and satisfies $j >_{P_{u,i}} j'$ {\em if and only if} there exists a directed path from $j$ to $j'$. Each DAG $P_{a,j}$ is defined analogously. Let us define the set of all partial ranks for $N$ users, and $K$ arms as $\mathcal{P}(N, K)$.
\end{definition}
\begin{definition}[Compatible Full Rank]
    A partial rank $(P'_{u}, P'_{a})$ is {\em compatible} with a partial rank $(P_{u}, P_{a})$ if for any pair of users $i, i' \in [N]$, and any pair of arms $j, j' \in [K]$, 
    $
    j \underset{P'_{u,i}}{>} j' \implies j' \underset{P_{u,i}}{\not>} j, \text{ and } i \underset{P'_{a,j}}{>} i' \implies i' \underset{P_{a,j}}{\not>} i.
    $\\
    The set of all full ranks {\em compatible} with a partial rank $(P_{u}, P_{a})$ is denoted as $\mathrm{FullRank}(P_{u}, P_{a})$.
\end{definition}
If a full rank $(F_{u}, F_{a})$ is {\em compatible} with a partial rank $(P_{u}, P_{a})$, then there exists a way of resolving incomparable pairs (by adding more directed edges) in the partial rank $(P_{u}, P_{a})$ to reach $(F_{u}, F_{a})$. The reversal of the process takes us from a full rank to a compatible partial rank. 

We now define the notion of super-stability for a partial rank, and follow it up with a proposition due to \cite{spieker1995set} that shows how it relates to the compatible full ranks. 
\begin{definition}[Super-Stable Match]
    A matching $M$ is called super-stable under a partial rank $(P_{u}, P_{a})$ if there is no pair user $i \in [N]$, and arm $j \in [K]$ such that simultaneously user $i$ does not strictly prefer her match $M(i)$ to arm $j$, and arm $j$ does not strictly prefer its match $M^{-1}(j)$ to user $i$, with $(i,j)\notin M$. The set of all super-stable matching of the partial rank $(P_{u}, P_{a})$ is denoted as $\mathrm{SuperStable}(P_{u}, P_{a})$.
\end{definition}

\begin{proposition}\label{prop:spieker}[Adapted from \cite{spieker1995set}]
    For a  partial ranking $(P_{u}, P_{a})$, the set of super-stable matching is the (possibly empty) intersection of all the stable matching of full ranks $(F_u, F_a)$ compatible with $(P_{u},P_{a})$, 
    \begin{align*}
    &\mathrm{SuperStable}(P_{u}, P_{a}) = \underset{(F_{u},F_{a}) \in \mathrm{FullRank}(P_{u}, P_{a})}{\cap} \mathrm{Stable}(F_{u}, F_{a}).
    \end{align*}
\end{proposition}
From the above proposition, given a full rank we define a set of partial match, namely {\em admissible partial rank}, for which a super-stable match is always stable under the specific full rank. 
\begin{definition}[Admissible Partial Rank]
    For a given full rank $(F_{u}, F_{a})$, we define a partial rank  $(P_u, P_a)$ to be \emph{admissible} if and only if $(F_{u}, F_{a}) \in \mathrm{FullRank}(P_{u}, P_{a})$ and $\mathrm{SuperStable}(P_{u}, P_{a}) \neq \emptyset$. We define the set of admissible partial rank instances as $\mathcal{A}(F_{u}, F_{a})$.
\end{definition}

It follows from Proposition~\ref{prop:spieker} that for any partial ranking in the set $\mathcal{A}(F_{u}, F_{a})$ a super-stable match is a stable match for the true Stable matching instance $(F_{u}, F_{a})$.
\begin{corollary}\label{corr:super-stable}
   For a full rank $(F_{u}, F_{a})$, for any admissible partial rank $(P_{u},P_{a}) \in \mathcal{A}(F_{u}, F_{a})$ and each super-stable matching
   $M \in \mathrm{SuperStable}(P_{u}, P_{a})$ we have $M \in Stable(F_{u}, F_{a}).$
\end{corollary}

\section{Centralized Two-Sided Matching Bandit}
In this section, we present a centralized algorithm that recovers a stable match with uncertainty of the rewards on both sides of the market. Each agent and arm maintain their own partial rankings based on UCB-LCB (similar to \cite{kong2023playeroptimal}). Similarly to \cite{liu_centralized_2020} in each round, users and arms share their UCB-LCB-based partial rankings with the centralized platform. The centralized platform then uses the partial ranking and finds a super-stable match if it exists using the Extended Gale-Shapley procedure for partial orders (Algorithm~\ref{alg:extended-GS}), which adapts the historical-proposal rule of Cseh and Juhos~\cite{cseh2021pairwise}. If such a match does not exist, then using a round-robin schedule, the centralized platform explores. 

\begin{algorithm}[ht!]
  \caption{Centralized Two-Sided Matching Bandit}
  \label{alg:centralized}
  \textbf{Exploration Index:} $\tau_{ex} \gets 0$ \\
  \textbf{Initial Ranking:} $P_{u,i}(1)\gets\text{empty $K$-node DAG}$ for all $i \in [N]$,  $P_{a,i}(1)\gets\text{empty $N$-node DAG}$ for all $i \in [K]$\\
  \For{$t\geq1$}{
    The centralized platform receives the partial ranks:\\
    \hspace{1em} users, $\mathcal{P}_{u}(t) = \{P_{u,i}(t) : i \in [N]\}$ and \\\hspace{1em} arms $\mathcal{P}_{a}(t) = \{P_{a, j}(t) : j \in [K]\}$.\\
    Obtain $M_{\rm stable} \gets \text{EXTENDED-GS}(\mathcal{P}_{u}(t), \mathcal{P}_{a}(t))$\\
    \If{$M_{\rm stable} = \emptyset$}{
      Play: $M(t) \gets \{m_i(t)=1+((i-1+\tau_{ex})\bmod K)\}$ \\
      Increase exploration index: $\tau_{ex} \gets \tau_{ex} + 1$
    }
    \Else{
      Play matching: $M(t)\gets M_{\rm stable}$
    }
    \For{each pair $(i,j) \in M(t)$}{
      User $i$ receives reward $Y_{i,j}(t)$, and arm $j$ receives reward $\tilde{Y}_{j,i}(t)$\\
      Update the empirical means and count for $(i,j)$
    }
    All users and arms recompute their confidence bounds and partial ranks for round $t+1$ (Eq.~\eqref{eq:pr})
  }
\end{algorithm}

We update the UCB, LCB, and partial rank estimates as follows for all $i \in [N]$ and all $j \in [K]$:
\begin{align}
    & \uucb_{i,j}(t) = \hat{\mu}_{i,j}(t) + \sqrt{6 \log(t) / n_{i,j}(t)}, \quad \ulcb_{i,j}(t) = \hat{\mu}_{i,j}(t) - \sqrt{6 \log(t) / n_{i,j}(t)} \label{eq:user-cb} \\
    & \aucb_{j, i}(t) = \hat{\gamma}_{j, i}(t) + \sqrt{6 \log(t) / n_{i,j}(t)}, \quad \alcb_{j, i}(t) = \hat{\gamma}_{j, i}(t) - \sqrt{6 \log(t) / n_{i,j}(t)} \label{eq:arm-cb}\\
    & P_{u,i}(t) = \begin{cases}
       j' > j  \text{ if } \uucb_{i,j}(t) < \ulcb_{i,j'}(t),\\
       j' < j  \text{ if } \uucb_{i,j'}(t) < \ulcb_{i,j}(t),\\
       j' \parallel j  \text{ otherwise.}
    \end{cases} 
    \quad\quad
    P_{a,j}(t) = \begin{cases}
       i' > i  \text{ if } \aucb_{j,i}(t) < \alcb_{j,i'}(t),\\
       i' < i  \text{ if } \aucb_{j,i'}(t) < \alcb_{j,i}(t),\\
       i' \parallel i  \text{ otherwise.}
    \end{cases} \label{eq:pr}
\end{align}
Here $\parallel$ denotes incomparability. At a fixed round, strict interval separation defines a strict partial order: $L_x>U_y$ and $L_y>U_z$ imply $L_x>U_y\geq L_y>U_z$. Overlapping intervals need not form transitive tie classes. Both algorithms are anytime: their confidence radii depend on the current round $t$, and $T$ is used only to state the regret bounds. All estimates and counts indexed by $t$ use observations strictly before round $t$; an unobserved pair has confidence interval $(-\infty,+\infty)$. Here for round $t$, for $i\in [N]$, arm $j \in [K]$, for user $i$ the mean reward for arm $j$ is $\hat{\mu}_{i,j}(t)$, for arm $j$ the mean reward for user $i$ is $\hat{\gamma}_{j, i}(t)$, and number of matches between user $i$ and arm $j$ is $n_{i,j}(t)$ as defined in the Appendix~\ref{app:algo}.

\subsection{Regret Upper Bound}
We obtain a logarithmic regret with two-sided uncertainty. Our result relies on a few key properties of the system. Firstly, we know from \cite{kong2023playeroptimal} that with high probability the LCB and UCB-based partial ranking contains the true ranking for each of the user and arm. Next, by Lemma~\ref{lemm:oracle_correctness}, the Extended Gale-Shapley procedure can retrieve a super-stable matching if there exists one under a given partial ranking, or declare if there is no such super-stable matching. Finally, the key observation is that a super-stable matching for a partial ranking is always a stable (not necessarily the user-optimal) matching for any full ranking contained by the specific partial ranking (c.f. \cite{spieker1995set}). Our exploration is forced when there is no super-stable matching available. 

We now define the gaps in our system that will appear in our regret upper bounds. Note that all the gaps we mention is with respect to the underlying rewards $(\boldsymbol{\mu}, \boldsymbol{\gamma})$. 
\begin{definition}[Minimum Gap]
    For a system with rewards $(\boldsymbol{\mu}, \boldsymbol{\gamma})$ and a partial rank $(P_{u}, P_{a})$, the minimum gap $\Delta_{\min}(P_{u}, P_{a}; \boldsymbol{\mu}, \boldsymbol{\gamma})$ is defined as the minimum gaps among the users and arms with different ranks, i.e.,
    \begin{align*}
        &\Delta_{\min}(P_{u}, P_{a}; \boldsymbol{\mu}, \boldsymbol{\gamma})
        = \min\big(\min\limits_{\substack{i \in [N],\\ j > j' \text{ in } P_{u,i}}} |\mu_{i,j} - \mu_{i, j'}|, \min\limits_{\substack{j \in [K],\\ i > i' \text{ in } P_{a,j}}} |\gamma_{j,i} - \gamma_{j, i'}|\big).
    \end{align*}
\end{definition}

\begin{definition}[Admissible Gap]
    Consider a system with rewards $(\boldsymbol{\mu}, \boldsymbol{\gamma})$ and corresponding full rank $(F_u, F_a)$. The admissible gap $\Delta_{\mathcal{A}}(\boldsymbol{\mu}, \boldsymbol{\gamma})$ is defined as the largest minimum gap for the admissible partial rankings of $(F_u, F_a)$, 
    $\Delta_{\mathcal{A}}(\boldsymbol{\mu}, \boldsymbol{\gamma}) = \max\limits_{(P_{u}, P_{a}) \in \mathcal{A}(F_{u}, F_{a})} \Delta_{\min}(P_{u}, P_{a}; \boldsymbol{\mu}, \boldsymbol{\gamma}).$
\end{definition}
Write $\Delta_{\min}$ for the minimum gap at the true full rank. We also need to define the reward differences between partners that are incomparable in a partial rank.
\begin{definition}[Overlap Width]
     For a system with rewards $(\boldsymbol{\mu}, \boldsymbol{\gamma})$ and a partial rank $(P_{u}, P_{a})$, the overlap width $\mathcal{W}_{\mathrm{ov}}(P_{u}, P_{a}; \boldsymbol{\mu}, \boldsymbol{\gamma})$ is defined as the maximum reward difference between incomparable partners, i.e.,
    \begin{align*}
        &\mathcal{W}_{\mathrm{ov}}(P_{u}, P_{a}; \boldsymbol{\mu}, \boldsymbol{\gamma}) = \max\big(\max\limits_{\substack{i \in [N],\\ j \parallel j' \text{ in } P_{u,i}}} |\mu_{i,j} - \mu_{i, j'}|, \max\limits_{\substack{j \in [K],\\ i \parallel i' \text{ in } P_{a,j}}} |\gamma_{j,i} - \gamma_{j, i'}|\big).
    \end{align*}
\end{definition}
We take a minimum over an empty set as $+\infty$ and a maximum of nonnegative gaps over an empty set as zero. We next show that the set $\mathcal{A}(F_{u}, F_{a})$ has the following `completeness' property with respect to reward gap. 
\begin{lemma}\label{lemm:structural-main}
    Consider the system with rewards $(\boldsymbol{\mu}, \boldsymbol{\gamma})$ and the corresponding full rank $(F_u, F_a)$. For a partial rank $(P_u, P_a)$ if $(F_u, F_a) \in \mathrm{FullRank}(P_u, P_a)$ and the overlap width $\mathcal{W}_{\mathrm{ov}}(P_u, P_a; \boldsymbol{\mu}, \boldsymbol{\gamma}) < \Delta_{\mathcal{A}}(\boldsymbol{\mu}, \boldsymbol{\gamma})$ then $(P_u, P_a) \in \mathcal{A}(F_u, F_a)$.
\end{lemma}

The regret upper bound proof shows that once we have enough exploration, such that we can resolve a gap of $\Theta(\Delta_{\mathcal{A}}(\boldsymbol{\mu}, \boldsymbol{\gamma}))$ with high probability in round $t$, the partial rank $(\mathcal{P}_{u}(t), \mathcal{P}_{a}(t))$ lies in the admissible set $\mathcal{A}(F^*_{u}, F^*_{a})$. On this event, the Algorithm~\ref{alg:centralized} always lands on a super-stable matching, as the set $\mathrm{SuperStable}(\mathcal{P}_{u}(t), \mathcal{P}_{a}(t)) \neq \emptyset$ as per Corollary~\ref{corr:super-stable}. Furthermore, due to Corollary~\ref{corr:super-stable} we know that $\mathrm{SuperStable}(\mathcal{P}_{u}(t), \mathcal{P}_{a}(t)) \subseteq \mathrm{Stable}(F^{*}_{u}, F^{*}_{a})$. Thus the algorithm incurs no stable regret in that round on this event. Finally, once the number of exploration satisfies $N_{explore}(t) > K + \frac{96K\log(T)}{\Delta^2_{\mathcal{A}}(\boldsymbol{\mu}, \boldsymbol{\gamma})}$ we show on the round-$t$ confidence event that $(\mathcal{P}_{u}(t), \mathcal{P}_{a}(t)) \in \mathcal{A}(F^*_{u}, F^*_{a})$, i.e. the recovered partial rank lies in the set of admissible partial rank instances of the true rankings. This culminates in the following regret bound.
\begin{theorem}[Centralized Regret Bound]
\label{thm:centralized_main}
    The cumulative regret of Algorithm~\ref{alg:centralized} in $T$ rounds for the true stable matching instance with no ties $(F^*_{u}, F^*_{a})$ satisfies the following upper bound. 
    \begin{gather*}
     \max\{\max_{i\in[N]}\mathbb E[R_{u,i}(T)],\max_{j\in[K]}\mathbb E[R_{a,j}(T)]\} \leq \mathbb{E}[R_{0/1}(T)],\quad
     \mathbb{E}[R_{0/1}(T)] \leq  
     \Big(K + \frac{2\pi^2 NK}{3} +  \frac{96 K\log(T)}{\Delta^2_{\mathcal{A}}(\boldsymbol{\mu}, \boldsymbol{\gamma})}\Big).
    \end{gather*}
\end{theorem}
 \paragraph{Improved Regret Bound:}
 Since the true full rank is itself admissible, $\Delta_{\min}\leq\Delta_{\mathcal A}$ holds for all the instances. For general instances, it is not possible to improve this relationship, as there are instances where they are equal. We now present a motivating example that shows stark separation between these two quantities.

 Consider an example with $N$
 users and $N$ arms, for some $N \geq 3$. Fix an $\varepsilon \in (0, 1/4)$. For each user $i \in [N]$ let $i$ and $(i+1) \mathrm{mod}\, N$ be the top $2$ arms with the respective rewards $\mu_{i,i} = (1 - \varepsilon)$ and $\mu_{i,(i+1) \mathrm{mod}\, N} = (1 - 2\varepsilon)$. The remaining arms $j \in [N]\setminus \{i, (i+1) \mathrm{mod} N\}$ all have mean reward $\mu_{i,j} = \varepsilon+\varepsilon^2 j/(N+1)$. For each arm $j \in [N]$, let  $\gamma_{j,(j - 1) \mathrm{mod}\, N} = \varepsilon$ and $\gamma_{j,i} = 1-\varepsilon+\varepsilon^2 i/(N+1)$ for all users $i \neq (j - 1) \mathrm{mod}\, N$. 
 All preferences are strict, and $\Delta_{\min}\leq\varepsilon$. Throughout this example, cyclic indices are interpreted in $[N]$. One stable matching for this instance is given as $\{(i,i): i \in [N]\}$.
 
 Next, let us consider the partial ranks for each user $i$,
 \begin{align*}
     P_{u,i} = \{i > j,\, &(i + 1)\mathrm{mod}\, N > j: \forall j \in [N] \setminus \{i, (i+1) \mathrm{mod}\,N\} \},
 \end{align*}
 and  the partial ranks for each arm $j$
 $$
 P_{a,j} = \{i > (j - 1)\mathrm{mod}\, N: \forall i \in [N] \setminus \{(j - 1) \mathrm{mod}\, N\} \}.
 $$
 For the above partial ranks, we have $\Delta_{\min}(P_u, P_a, \boldsymbol{\mu}, \boldsymbol{\gamma}) \geq 1-3\varepsilon-\varepsilon^2$. Moreover, this partial rank $(P_u, P_a)$ has $\{(i,i): i \in [N]\}$
 as a super-stable matching. Therefore, $(P_u, P_a)$ is admissible, and the gap $\Delta_{\mathcal{A}} \geq 1-3\varepsilon-\varepsilon^2$. The ratio $\tfrac{\Delta_{\mathcal{A}}}{\Delta_{\min}} \geq \tfrac{1-3\varepsilon-\varepsilon^2}{\varepsilon}$ is unbounded as $\varepsilon \to 0$. This shows that there can be an arbitrary separation between $\Delta_{\min}$ and $\Delta_{\mathcal{A}}$.

\section{Decentralized Two-Sided Matching Bandit}
In this section, we convert the centralized algorithm to a decentralized version where the users use two global binary flags $restart$ and $success$ to coordinate. We present Algorithm~\ref{alg:decentralizedUser} used by a user $i$, and Algorithm~\ref{alg:decentralizedArm} used by an arm $j$ in the Appendix~\ref{app:algo} due to the lack of space.

Our algorithm proceeds in phases (of at most $H=N^2+1$ steps each) and simulates the Algorithm~\ref{alg:extended-GS} in each phase in a staggered manner. An ongoing phase is terminated by setting the $restart$ flag to True by a user, and as the flag $restart$ is shared the change of phase is always in sync across arms and users.  Each phase $\tau_{\ell}$, starting at time $t + 1$, begins with setting the initial partial rank of each user $i$ to $P_{u,i}(t+1)$ and each arm $j$ to $P_{a,j}(t+1)$ (the UCB-LCB based partial ranks obtained at the end of previous phase). This is similar to the centralized algorithm starting a round with updated partial ranks.  

Next, a user $i$ reads the flag $success$ which if False indicates the phase $\tau_{\ell}$ is used for exploration.  Otherwise, the $success$ flag being True indicates the previous phase ended with a super-stable match. This gives each user a unique certified partner, which is used as $M_i(\tau_\ell)$ throughout the next phase. The arms likewise retain their preceding certified partners for physical play.

Inside the phase $\tau_{\ell}$ in each round, we have the following alternating steps between users and arms.
\\$(i)$\,\,\textbf{A user $i$} first propose to the `source' nodes of the (updated) partial rank $PR_{u,i}(\tau_{\ell})$.
\\$(ii)$\,\,\textbf{An arm $j$} adds the current proposers to its phase history and computes the source nodes in the frozen arm order restricted to that history. It deletes strictly dominated active users. If there is a unique historical source, it accepts that user; otherwise, it deletes the maximal historical proposers and rejects the current proposals. Rejection notifications are also sent for newly deleted pairs that have not proposed. Rejected users remain in the history, as in Algorithm~\ref{alg:extended-GS}.
\\$(iii)$\,\,\textbf{A user $i$} receives the accept or reject signals from the arms, and deletes rejecting arms from $PR_{u, i}(\tau_{\ell})$. Next, it determines whether to explore or exploit based on $explore(\tau_{\ell})$. If it explores then matches with the arm $m_i(t) = 1+((i-1+\tau_{ex})\bmod K)$, otherwise matches with $m_i(t) = M_i(\tau_{\ell})$ . We note that $\tau_{ex}$ is also in sync for each user as $explore(\tau_{\ell})$ and phases are in sync. Next, the user and the arms, if matched, observe their respective rewards, and updates the partial ranks $P_{u,i}(t+1)$ for users $i \in [N]$, and $P_{a, j}(t+1)$ for arms $j \in [K]$.

In the \emph{global communication} stage, the users check whether the iteration deleted a pair or emptied a user list. A deletion-free iteration triggers the terminal matching test; an empty list or reaching the phase cap while deletions continue gives failure. The ordered two-flag implementation, including the check that every previously proposed-to arm is held, is specified in Appendix~\ref{app:algo}. Finally, all users and arms read $restart$ and enter the next phase together. Proposal acceptances during a phase do not change its physical matching, which is fixed by the preceding phase's outcome.

Our regret upper bound for the decentralized system follows the centralized system closely. We argue that each phase mimics one round of the centralized system.  After $O\big(K\log(T) / \Delta_{\mathcal{A}}^2(\boldsymbol{\mu}, \boldsymbol{\gamma})\big)$ many rounds of exploration, any phase with high probability ends in finding a true stable matching, similar to Lemma~\ref{lemm:explore}. The final regret bound in the decentralized case is given as follows.

\begin{theorem}[Decentralized Regret Bound]
\label{thm:decentralized_main}
     For a true stable matching instance $(F^*_{u}, F^*_{a})$ when the users follow Algorithm~\ref{alg:decentralizedUser}, and the arms follow Algorithm~\ref{alg:decentralizedArm}, the cumulative regret in $T$ rounds  satisfies the following upper bound. 
    \begin{gather*}
     \max\{\max_{i\in[N]}\mathbb E[R_{u,i}(T)],\max_{j\in[K]}\mathbb E[R_{a,j}(T)]\} \leq \mathbb{E}[R_{0/1}(T)],\\
     \mathbb{E}[R_{0/1}(T)] \leq  
     \Big(2(N^2+1) + K + \frac{2\pi^2 NK(N^2+1)}{3} +  \frac{96 K\log(T)}{\Delta^2_{\mathcal{A}}(\boldsymbol{\mu}, \boldsymbol{\gamma})}\Big).
    \end{gather*}
\end{theorem}

\textbf{Balancing Communication and Regret:} We use two reusable shared flags together with local proposal/reply messages. The bound assumes that the ordered flag accesses fit within a round's communication stage. If these operations instead consume reward rounds, their duration must be included in the regret analysis; the theorem does not give the same additive constant in that model.

\textbf{Experimental Validation:}  We numerically compare Extended-GS with Explore-then-Gale Shapley (ETGS). Due to lack of space we defer the details of our experiment in Appendix~\ref{app:experiment}. We omit decentralized simulations; its theoretical upper bound differs by an time-independent additive $O(N^3K)$ under the stated communication model. We compare against centralized Explore-then-Gale Shapley (ETGS) algorithm for fair comparison.\footnote{A decentralized version of this algorithm with shared global flags (aka blackboard) is presented in Algorithm 2 in \cite{tejas_2024}.} Extended-GS significantly outperforms ETGS by quickly identifying a viable partial ranking during the exploration phase. This leads to lower regret than ETGS, which needs a full ranking of the top N items before achieving comparable performance. 

\textbf{Regret Comparison:}
We end this section by comparing the regret guarantees of a selected few works in this domain in Table~\ref{tab:regret_bounds_long}. We provide the first regret upper bound that works for two sided general matching markets and depends on the instance dependent intrinsic gap $\Delta_{\mathcal{A}}$. Our decentralized guarantee uses the shared-flag communication model stated in Section~\ref{sec:problem}. We also provide a centralized binary stable regret lower bound through the information constraints associated with admissible partial ranks. See Section~\ref{sec:lower} for details.

\begin{table*}[htb]
\centering
\small
\begin{tabular}{ |c|c|c|c| } 
 \hline
 \textbf{Algorithm} & \textbf{Two sided} & \textbf{Market Assumptions}  & \textbf{Upper Bound and Metric}   \\ 
 \hline
 C-UCB~\cite{liu_centralized_2020} & NO & Centralized & Pessimal $O(NK\log(T)/ \Delta_{\min}^2)$ \\ 
 \hline
 UCB-D3~\cite{ucbd3} & NO & Serial Dictatorship  & Unique  $O(NK\log(T)/ \Delta_{\min}^2)$  \\
 \hline
 CA-UCB~\cite{liu_decentralized_2021} & NO & (user, arm) broadcast  & Pessimal $O(\exp(N)N^5K^2\log(T)/ \Delta_{\min}^2)$  \\ 
  \hline
 UCB-D4~\cite{basu_2021} & NO & uniqueness consistency  & Unique $O(NK\log(T)/ \Delta_{\min}^2)$ \\
 \hline
 UCB-DMA~\cite{maheshwari_2022} & NO & $\alpha$-reducible  & Unique $O(\mathcal{C}_{\alpha}NK\log(T)/ \Delta_{\min}^2)$  \\
  \hline
 ETGS~\cite{kong2023playeroptimal} & NO & (user, arm) broadcast & Optimal $O(K\log(T)/ \Delta_{\min}^2)$ \\ 
 \hline
 PCA-DAA~\cite{pokharel2023converging} & YES & (user, arm) broadcast  & -  \\
 \hline
 ETGS+BB~\cite{tejas_2024} & YES & (user, arm) broadcast  & User Optimal  $O(K\log(T)/ \Delta^2_{\min}))$ \\
 \hline
 CA-ETC~\cite{tejas_2024} & YES & no broadcast  & User Optimal  $O(poly(T))$ \\
 \hline
 Ours & YES & shared flags  & Pessimal \& Binary   $O(K\log(T)/ \Delta_{\mathcal{A}}^2)$  \\
 \hline
\end{tabular}
\begin{tabular}{ |c|c|c| }
\hline
\textbf{Two sided} & \textbf{Market Type} & \textbf{Lower Bound and Metric}\\
\hline
NO & No Broadcast   & Unique $\Omega(N\log(T)/ \Delta_{\min}^2 + K\log(T)/ \Delta_{\min})$~\cite{ucbd3} \\
\hline
YES & Centralized & Binary $\liminf R/\log T\geq c(\theta)$ [Ours] \\
\hline
\end{tabular}
 \caption{A more comprehensive regret comparison. The upper-bound gaps satisfy $\Delta_{\min}\leq\Delta_{\mathcal A}$. Our decentralized bound includes an time-independent additive $O(N^3K)$ under the stated communication model. $C_{\alpha}$ is a parameter specific to $\alpha$-irreducibility  in \cite{maheshwari_2022}. 
 Our lower bound concerns strict interior Bernoulli instances and the information constant in Theorem~\ref{thm:centralized_lower}. In regret upper bound, `Unique' means unique stable matching, `Optimal' means User optimal stable regret, `Pessimal' means User pessimal stable regret, and `Binary' means binary stable regret. 
 }
 \label{tab:regret_bounds_long}
 \vspace{-1.5em}
\end{table*}

\section{Regret Lower Bound}\label{sec:lower}
We develop an instance-dependent lower bound for {\em binary stable regret} for the centralized setting.  A pessimal stable regret lower bound in general instances is not meaningful for the learning task of finding the stable matching. For example, there may exists a set of matchings which are not stable but can be scheduled to achieve negative pessimal stable regret. Hence, {\em binary stable regret} which is always nonnegative is the right quantity to lower bound for the bandit learning problem in matching markets.  

We adapt the framework in \cite{combes2017minimal} to our multi-agent setup. The detailed formulation is provided in the Appendix~\ref{app:lower}. Our system is parameterized by $\theta = (\boldsymbol{\mu}, \boldsymbol{\gamma}) \in  \Theta := \{\theta\in(0,1)^{2NK}:\text{all within-agent mean comparisons are strict}\}$. We focus on independent Bernoulli rewards. Write $F(\lambda)=(F_u(\lambda),F_a(\lambda))$ for the strict full preference profile induced by $\lambda\in\Theta$; a numerical parameter is compatible with a partial rank when its induced full profile extends that rank. For any $(i,j) \in [N]\times [K]$, the term $kl(\theta, \lambda; (i,j))$ denotes the sum of KL-divergence between the $\theta$ and $\lambda$ instances for the rewards associated with the match $(i,j)$ and satisfies
$
kl(\theta, \lambda; (i,j)) = kl(\mu_{i,j}(\theta), \mu_{i,j}(\lambda)) + kl(\gamma_{j,i}(\theta), \gamma_{j,i}(\lambda)).
$ This can be extended to the $kl$ divergence given a matching $\mathcal{M}$ as  
$
kl(\theta, \lambda; \mathcal{M}) = \sum_{(i,j)\in \mathcal{M}} kl(\theta, \lambda; (i,j)).
$

We need to take extra care while adapting the information lower-bound framework because there may be multiple stable matchings. The direct change-of-measure proof in Appendix~\ref{sec:proof_lower_bound} uses alternatives under which every true stable matching becomes unstable, while observations on those matchings have unchanged laws. Therefore, our `bad' parameter space for a given $\theta$ becomes \begin{align*}
 \Lambda(\theta) = \{\lambda \in \Theta: &\underbrace{ kl(\theta, \lambda; \mathcal{M}) = 0, \forall \mathcal{M} \in \mathrm{Stable}(\theta)}_{\text{for all stable match of $\theta$ divergence is $0$}}; \underbrace{\mathrm{Stable}(\lambda) \cap \mathrm{Stable}(\theta) = \emptyset}_{\text{no common solution}}\}. 
\end{align*}
We now relate this alternative set to the admissible sets. If $\mathrm{Stable}(\lambda)\cap\mathrm{Stable}(\theta)=\emptyset$, no partial rank compatible with both profiles admits a super-stable matching. The converse need not hold: a shared stable matching may be protected against a given outside pair by different endpoints in the two instances, so no common partial rank certifies it. Consequently,
\[
\bigcup_{P\in\mathcal A(\theta)}\mathrm{FullRank}(P)
\subseteq\{F(\lambda):\mathrm{Stable}(\lambda)\cap\mathrm{Stable}(\theta)\ne\emptyset\},
\]
with possible strict inclusion. The information constraints, rather than the alternative sets themselves, will be shown to coincide.
As there can be no divergence for any matched pairs in a true stable matching that narrows the `bad' set of parameters further. We first define the `Locked' edges as follows.
\begin{definition}\label{def:locked_edges}
For any instance $\theta \in \Theta$, we define $\mathrm{Locked}(\theta)$ as the set of all user-arm pairs $(i,j)$, such that $(i,j) \in \mathcal{M}$ for some match $\mathcal{M} \in Stable(\theta)$.
\end{definition}
Define the admissible-rank alternative family by 
\begin{align}
&\Lambda_{\mathcal A}(\theta) = \big\{\lambda \in \Theta: F(\lambda) \notin\!\!\!\!\! \underset{(P_u, P_a) \in \mathcal{A}(\theta)}{\cup}\!\!\!\!\!\!\!\mathrm{FullRank}(P_u, P_a);\forall (i,j)  \in \mathrm{Locked}(\theta): \mu_{ij}(\theta) =  \mu_{ij}(\lambda), \gamma_{ji}(\theta) = \gamma_{ji}(\lambda)  \big\}.\label{eq:bad_set}
\end{align}

Lemma~\ref{lemm:alternative_domination} proves that $\Lambda(\theta)\subseteq\Lambda_{\mathcal A}(\theta)$ and that both families impose the same constraints on nonnegative observation allocations. A policy $\pi$ is {\em uniformly good} if for any  $\theta = (\boldsymbol{\mu}(\theta), \boldsymbol{\gamma}(\theta))$ the algorithm $\pi$ has $O(T^{\alpha})$ expected binary stable regret for any $\alpha > 0$. We now state our regret lower bound below for any uniformly good policy.

\begin{theorem}[Centralized Regret Lower Bound]\label{thm:centralized_lower}
For any instance $\theta \in \Theta$, for any {\em uniformly good} policy $\pi$ the binary stable regret is lower bounded as 
$\lim\inf_{T\to \infty} \frac{\mathbb{E}[R_{0/1}^{\pi}(T; \theta)]}{\log(T)} \geq c(\theta)$, where $c(\theta)$ minimizes the following optimization problem 
\begin{align*}
    &\min_{\eta(M) \geq 0} 
    \sum_{M \in \mathrm{Match}(N,K)\setminus \mathrm{Stable}(\theta)}\eta(M), \text{ s.t.} \sum_{(i, j)\in [N]\times[K]}\sum_{M\ni (i,j)}\eta(M) kl(\theta, \lambda; (i,j))  \geq 1, \forall \lambda \in \Lambda_{\mathcal A}(\theta),
\end{align*}
where $\eta(M)=0$ for stable matchings, without loss of generality, and $\Lambda_{\mathcal A}(\theta)$ is defined in Equation~\eqref{eq:bad_set}. 
\end{theorem} 

We now present binary stable regret lower bounds derived for some special instances.
\begin{corollary}[Serial Dictatorship]\label{cor:serial_lower}
Consider a general serial dictatorship instance $\theta$ with two sided uncertainty, where $N\leq K$, $\gamma_{j,i}(\theta) > \gamma_{j,i'}(\theta)$ for all $i < i'$ and $j$, and the unique stable matching is $\{(i,i): i \in [N]\}$. Then the binary stable regret for $\theta$ is lower bounded as for any {\em uniformly good} policy $\pi$ is lower bounded as 
$$
\lim\!\!\!\inf_{T\to \infty} \tfrac{\mathbb{E}[R_{0/1}^{\pi}(T; \theta)]}{\log(T)} \geq \max\left(\max_{i \in [N]}c_{u,i}(\theta), \max_{j \in [K]}c_{a,j}(\theta)\right)
$$
where $D_i(\theta)=\{j<i:\mu_{i,j}(\theta)>\mu_{i,i}(\theta)\}$, and
\begin{align*}
 &c_{u,i}(\theta) = \sum_{j\in D_i(\theta)}\tfrac{1}{kl(\gamma_{j,i}(\theta), \gamma_{j,j}(\theta))} + \sum_{j > i} \tfrac{1}{kl(\mu_{i,j}(\theta), \mu_{i,i}(\theta))},\\
 &c_{a,j}(\theta) =\sum_{i: j \in D_i(\theta)}\tfrac{1}{kl(\gamma_{j,i}(\theta), \gamma_{j,j}(\theta))} + \sum_{i< j} \tfrac{1}{kl(\mu_{i,j}(\theta), \mu_{i,i}(\theta))}.
\end{align*}
\end{corollary}
Note the above instance holds with out loss of generality up to renaming of arms and user, and thus covers all instances of general serial dictatorship. This lower bound differs from \cite{ucbd3} in several ways.  It's for binary stable regret (not per-user regret), applies to general serial matching (not just OSB instances), and accounts for two-sided (not just user-sided) uncertainty. Thus, the bounds are incomparable.

\begin{corollary}[General Instance with Redundant Arm]\label{cor:redundant_lower}
Consider a general two-sided stable matching instance $\theta$ with redundant arms $N < K$,  where $j^*(i) \triangleq   \arg\max_{j}\{ \mu_{i,j}(\theta): (i,j) \in Locked(\theta)\}$. Then the binary stable regret for $\theta$ is lower bounded as for any {\em uniformly good} policy $\pi$ is lower bounded as 
$\lim\inf_{T\to \infty} \tfrac{\mathbb{E}[R_{0/1}^{\pi}(T; \theta)]}{\log(T)} \geq \hfill$  $$\max\left(\max_{i \in [N]}c_{u,i}(\theta), 
 \max_{j: (\cdot, j) \notin Locked(\theta)}c_{a,j}(\theta)\right) \text{ where,}
$$
\begin{align*}
 &c_{u,i}(\theta) = \sum_{j: (\cdot,j) \notin Locked(\theta)} \tfrac{1}{kl(\mu_{i,j}(\theta), \mu_{i,j^*(i)}(\theta))},\\
 &c_{a,j}(\theta) =\sum_{i \in [N]} \tfrac{1}{kl(\mu_{i,j}(\theta), \mu_{i,j^*(i)}(\theta))}.
\end{align*}
\end{corollary}

\textbf{General Instances:} Because of space constraints, detailed definitions and results for the fully general case (where N may equal K) are relegated to Appendix~\ref{sec:proof_lower_bound}.  We provide a brief and high-level overview here. 

A {\em cover} of the admissible ranks selects at least one strict comparison from each inclusion-minimal admissible rank. Reversing those comparisons while preserving all locked means produces admissible-rank alternatives, provided the reversals are jointly feasible. The precise definition appears in Appendix~\ref{app:lower}.

Theorem~\ref{thm:centralized_lower_final} yields
\begin{align}\label{eq:gen_c_theta}
c(\theta)\geq\max_{G\in\mathcal{CG}(\theta)}\sum_{C\in G}kl(\theta,C)^{-1}.
\end{align}
Here $\mathcal{CG}(\theta)$ consists of feasible cover groups whose changed-edge supports cannot be observed together in a single matching, and
\[
kl(\theta,C)=\inf_{\lambda\in\Lambda_C}\max_{M\notin\mathrm{Stable}(\theta)}kl(\theta,\lambda;M).
\]
The set $\Lambda_C$ enforces the selected reversals and fixes means outside the cover's unlocked support. Mere disjointness of changed-edge supports is insufficient for adding their contributions, since one matching can observe several disjoint edges. The quantity $kl(\theta,C)$ retains the joint feasibility of the reversals; the one-edge alternatives above give explicit KL bounds for the stated special cases.

\paragraph{Remarks on Lower Bound:} Multiple remarks on our regret lower bound is in order.

$\diamond\,\,$\textbf{OSSB Applicability:} We note that in \cite{combes2017minimal} the authors provide an algorithm that can attain the regret lower bound named OSSB. However, as the size of the set of all matchings is exponential in $N$) a OSSB style algorithm is computationally prohibitive for our problem. More importantly, OSSB algorithm is designed for instances with unique optimal solution, so it is unclear if OSSB can be applied in our case directly.
\\$\diamond\,\,$\textbf{Regret `Free' Instances:} If every pair $(i,j)$ occurs in some $\theta$-stable matching, then every pair is locked, $\Lambda(\theta)=\emptyset$, and $c(\theta)=0$. For each such target instance there also exists a policy, uniformly good on $\Theta$, with $O_\theta(1)$ expected binary regret at that instance. The policy and its validation step are given in the proof of Corollary~\ref{corr:wide_sub_log}; this is an existence result and does not assert bounded regret for our proposed algorithm. For $N=2$, in the instance $U_1:[1,2],U_2:[2,1],A_1:[2,1],A_2:[1,2]$, both $[(1,1),(2,2)]$ and $[(2,1),(1,2)]$ are stable and together cover all edges.

\section{Conclusion and Future Work}
This work investigates two-sided bandit learning in general matching markets, providing both centralized and decentralized algorithms with logarithmic regret guarantees under the stated communication model (Theorems~\ref{thm:centralized_main} and \ref{thm:decentralized_main}, respectively). Our algorithms use super-stability~\cite{irving1994stable} and a simultaneous adaptation of the partial-order proposal algorithm of Cseh and Juhos~\cite{cseh2021pairwise}, playing a certified matching when possible and exploring otherwise.  Our centralized binary stable regret lower bound for strict interior Bernoulli instances (Theorem~\ref{thm:centralized_lower_final}) highlights the importance of the super-stable set in determining the difficulty of stable matching bandit learning.  While our current work employs a sub-optimal round-robin exploration strategy, future work will focus on refining exploration techniques to achieve tight regret bounds.  Furthermore, while we guarantee pessimal stable regret, developing algorithms that optimize for specific objectives (e.g., user-optimal or social-optimal) represents another important direction. We hypothesize that our initial stable matching algorithm could be combined with the distributive lattice structure of the stable matching set to efficiently explore and identify optimal stable matching.

\bibliographystyle{unsrt}  
\bibliography{ref}  
\newpage
\appendix
\section{Related Work}
\paragraph{Stable Matching with Partial Preference:} Stable matching has been a successful concept for studying equilibria among incentivized agents in two-sided matching markets. In the foundational work \cite{gale1962college}, participants express preferences through strict rankings. A natural extension allows for more realistic scenarios with partial preferences. \cite{irving1994stable} expanded the concept of stability to partial preferences by introducing three distinct stability concepts: weak, strong, and super-stability. Our work connects super-stability to the inherent hardness of learning stable matchings under bandit feedback. For partial orders, Cseh and Juhos~\cite{cseh2021pairwise} give a proposal algorithm whose acceptance rule compares a new proposer with all historical proposers, including rejected ones. Their terminal test requires a matching covering every recipient that received a proposal. Our Extended-GS procedure adapts these principles to simultaneous proposals and deletion of strictly dominated active pairs; Lemma~\ref{lemm:oracle_correctness} proves correctness and iteration bounds for this implementation. Next, we build on the work of \cite{spieker1995set} which provides structure of the set of super-stable matchings. 
\paragraph{Bandit Learning in Matching Markets:} The work of \cite{das2005two} introduced the problem of learning stable matching for a matching markets with one-sided uncertainty, i.e. the users have unknown preferences but the arms have complete knowledge of their prefrences. In \cite{liu_centralized_2020} the authors provided a UCB based centralized algorithm with the first theoretical regret guarantee of $O(NK\log(T)/ \Delta_{\min}^2)$. An explore-then-commit algorithm with the knowledge of minimum gap $\Delta_{\min}$ was also proposed. Next, the work was extended to the decentralized setup that removes the knowledge of gap but imposes uniqueness of stable matching via structural assumptions on preference rank~\cite{ucbd3, basu_2021, maheshwari_2022}. The algorithms proposed therein achieve a $O(NK\log(T)/ \Delta_{\min}^2)$ user pessimal stable regret. The collision avoidance UCB based algorithm proposed in \cite{liu_decentralized_2021} achieves $O(\exp(N)\log^2(T)/ \Delta_{\min}^2)$ regret in general matching markets. 

A separate thread of work explored the Explore-then-commit type algorithms for general markets starting from \cite{basu_2021}, and improved in the ETGS \cite{kong2023playeroptimal}. The ETGS type algorithms are able to achieve a $O(K\log(T)/ \Delta_{\min}^2)$ user optimal stable regret. A more recent work \cite{hosseini2024putting} provides sample complexity bounds, which translates to $O(E\log(T)/ \Delta_{\min}^2)$ regret bound, using a concept of envy-set where $E$ is the size of average envy set, and may be less than $K$ based on the underlying instances. There has been multiple works with ETGS as the core where a user can match with multiple arms (upto a quota) in each round \cite{kong2024improved, saha2024altruistic}. 

Only a select few recent works \cite{pokharel2023converging, tejas_2024, zhangdecentralized} study matching markets with two-sided uncertainty where both sides need to learn their respective preferences from bandit feedback. In \cite{tejas_2024}, the authors present a black-board model similar to ours and attain a $O(K\log(T)/ \Delta_{\min}^2)$ algorithm similar to centralized ETGS mentioned in this work. Similar guarantees are provided in \cite{zhangdecentralized} without baclkboard, while \cite{pokharel2023converging} lacks any theoretical guarantee. But all these works focus on user optimal stable match learning. This is incomparable to the task in this paper of finding any one stable matching. 

Our algorithmic contribution is to use super-stable matching under partial preferences~\cite{irving1994stable,cseh2021pairwise} to guide exploration and exploitation in two-sided matching bandits. Furthermore, we provide a new instance dependent lower bound for general instances for centralized setting, and hence also for decentralized setting. The long standing regret lower bound before this work comes from \cite{basu_2021} which is applicable only to decentralized setting, and for the limited special case of serial dictatorship.

On the algorithmic side, apart from Gale Shapley based bandit algorithms, the authors in \cite{srikanth_uiuc} take a new game-theoretic approach of preference learning with exponential weight learning with one-sided uncertainty and provide a $O(N^2log^3(T) +  N^2log(T)/ \Delta_{\min}^N)$ regret bounds. Also, there has been work that involved non-stationarity in the mean rewards of the arms \cite{avishek_nonstationary, maheshwari_2022}, and mean rewards have a linear contextual structure \cite{jagadeesan2021learning, parikh2024competing}. Finally, in \cite{lin2024stable} the authors study approximation ratios and learning to find weak-stable matches which is different from the standard setup.

\section{Experiments}\label{app:experiment}
In this section, we numerically compare Extended-GS with Explore-then-Gale Shapley (ETGS). We omit decentralized simulations. We first show the pessimal stable regret dynamics of the centralized algorithm for a system with $N$ users and $K$ arms. We generate a bandit instance by selecting a random preference order for each user and each arm. We generate the mean rewards $\mu_{i,j}$ and $\gamma_{j,i}$, for $i \in [N]$ and $j \in [K]$ such that the permutation is randomized, and the gaps are chosen uniformly at random from  $[\Delta_{\min}, \Delta_{\max}]$. For each reward observation we apply independent Gaussian noise with standard deviation $\sigma$ (the UCB and LCB bonuses are also scaled with $\sigma$). We plot the evolution of the mean cumulative regret with time for Extended-GS and a centralized Explore-then-Gale Shapley (ETGS) algorithm of Kong et al.~\cite{kong2023playeroptimal} for fair comparison.\footnote{A decentralized version of this algorithm with shared global flags (aka blackboard) is presented in Algorithm 2 in \cite{tejas_2024}.}  Centralized ETGS explores until the top $N$ ranks are resolved for each user and arms, and then commits to the Gale Shapley solution afterwards. 

In Figure~\ref{fig:evolution}, we consider one such instance with  $N=8$ users, $M=8$ arms, $\Delta_{\min} = 0.2, \Delta_{\max} = 0.5$, and $\sigma = 0.4$. We observe the logarithmic pessimal stable regret for both the algorithms. However, clearly the proposed Extended-GS based algorithm outperforms the Centralized ETGS adapted from Kong et al.~\cite{kong2023playeroptimal}. We plot the mean rewards, as well as $25\%$ and $75\%$ regret plots averaged over $20$ sample paths. Next, in Figrue~\ref{fig:regret_stats} we study the dependence on $\sigma / \Delta_{avg}$ by varying $\sigma$ for $N=K=5$ and $\Delta_{\min} = 0.2, \Delta_{\max} = 0.5$.  We observe that there is a clear separation between the Centralized ETGS and proposed Extended-GS algorithm showing the superiority of the latter. Both the algorithms initially explore but Extended-GS algorithm finds an admissible partial rank much earlier than the top $N$ rank resolution leading to a lower regret than ETGS. Figrue~\ref{fig:regret_gap} exhibits similar behavior with varying gaps $\Delta_{\min}$ and $\Delta_{\max}$.

\begin{figure}[h]
\centering
\includegraphics[width=0.55\linewidth]{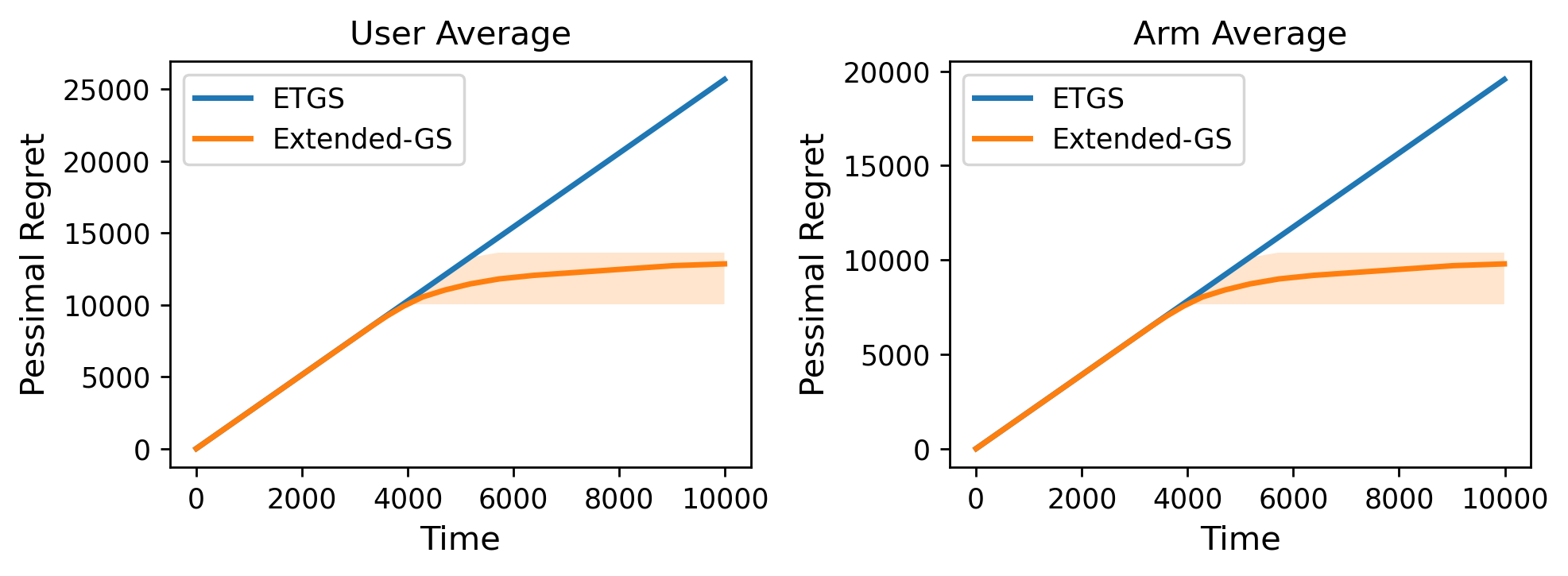}
\caption{Pessimal Stable Regret for $N=8$ Users, $M=8$ Arms, $\Delta_{\min} =0.2$, $\Delta_{\max} = 0.5$, and $\sigma=0.5$.}
\label{fig:evolution}
\end{figure}

\begin{figure}[h]
\centering
\includegraphics[width=0.55\linewidth]{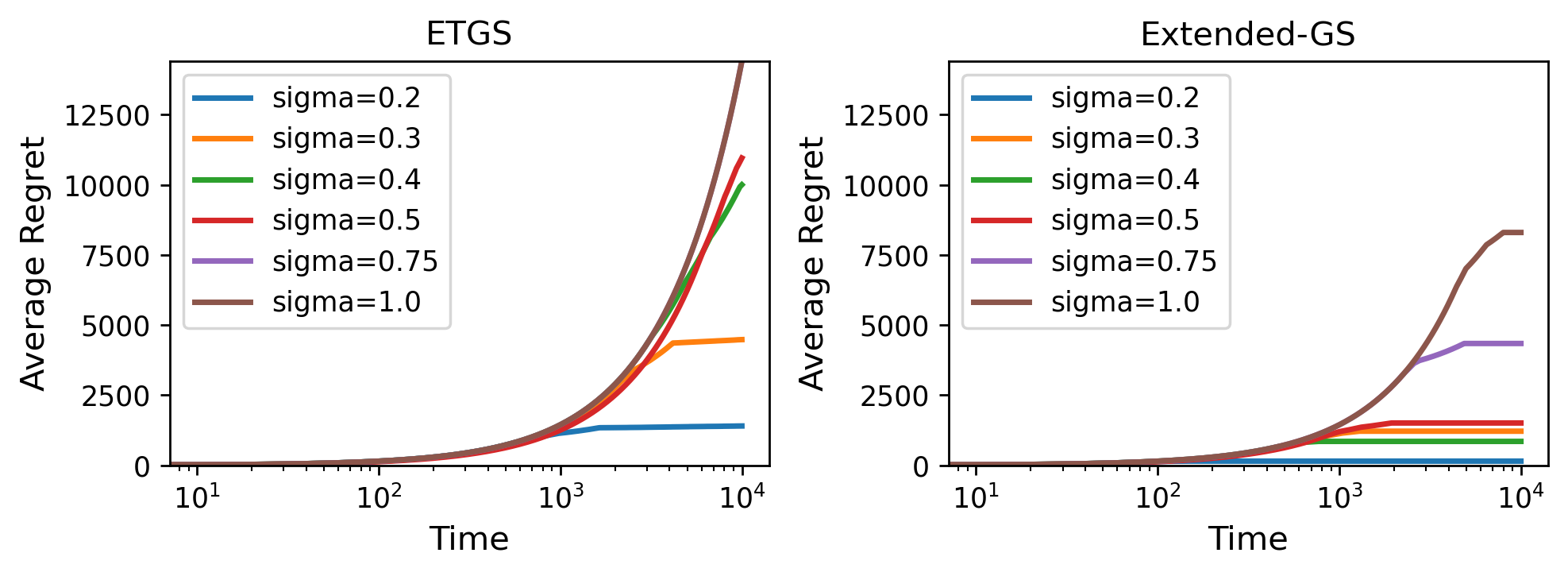}
\caption{Average pessimal stable regret across all users and arms with different $\sigma$, with $N=5$ Users, $M=5$ Arms, $\Delta_{\min} =0.2$, $\Delta_{\max} = 0.5$.}
\label{fig:regret_stats}
\end{figure}

\begin{figure}[h]
\centering
\includegraphics[width=0.55\linewidth]{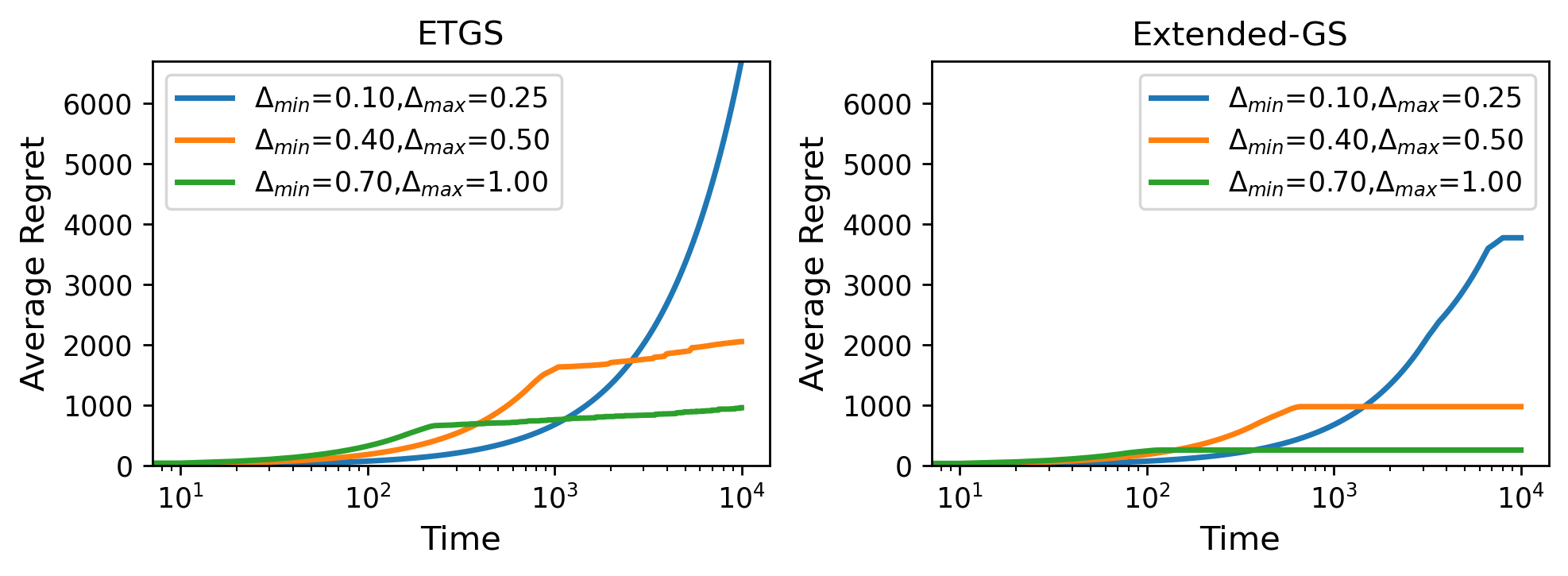}
\caption{Average pessimal stable regret across all users and arms with different $\Delta_{\min}$, and $\Delta_{\max}$, with $N=5$ Users, $M=5$ Arms, $\sigma = 0.4$.}
\label{fig:regret_gap}
\end{figure}
\vspace{-0.5em}

\section{Algorithms Addendum}\label{app:algo}
In round $t$ for $i \in [N]$ and $j \in [K]$, the estimated mean rewards $\hat{\mu}_{i,j}(t)$ and $\hat{\gamma}_{j, i}(t)$, and number of samples $n_{i,j}(t)$ for our Algorithms (Algo~\ref{alg:centralized},Algo~\ref{alg:decentralizedArm}, and Algo~\ref{alg:decentralizedUser}) use observations before round $t$, with $n_{i,j}(1)=0$, and are updated after that round as (empirical means are initialized to zero)
\begin{align*}
    &n_{i,j}(t+1) = n_{i,j}(t) + \mathbbm{1}((i,j) \in M(t)) \\
    &\hat{\mu}_{i,j}(t+1) = \frac{n_{i,j}(t) \hat{\mu}_{i,j}(t) + Y_{i}(t)} {n_{i,j}(t) + 1} \text{ if } (i,j) \in M(t) \text{ else } \hat{\mu}_{i,j}(t) \notag  \\
    &\hat{\gamma}_{j, i}(t+1) = \frac{n_{i,j}(t) \hat{\gamma}_{j, i}(t) + \tilde{Y}_{j}(t)} {n_{i,j}(t) + 1} \text{ if } (i,j) \in M(t) \text{ else } \hat{\gamma}_{j,i}(t).
\end{align*}

We now state the Extended-GS procedure for partial orders. Its historical-proposal acceptance rule and requirement to cover every arm that received a proposal are adapted from Algorithm~3 of Cseh and Juhos~\cite{cseh2021pairwise}. Their algorithm processes individual proposals; our implementation uses simultaneous proposals and deletes strictly dominated active pairs. Lemma~\ref{lemm:oracle_correctness} gives a self-contained proof of correctness and the iteration bounds for this implementation. Source and sink refer to the preference order restricted to the active vertices, with edges directed from better to worse partners. Removing a vertex preserves all comparisons between remaining vertices (equivalently, restrict the transitive closure). Each arm also retains the frozen input order and its proposal history, including rejected proposers. All historical and deletion comparisons are evaluated in the frozen input order; active copies only record which pairs remain eligible.
\begin{algorithm}[!htbp]
  \caption{Extended Gale Shapley (EXTENDED-GS)}
  \label{alg:extended-GS}
  \textbf{Input:} User partial rankings $\mathcal{P}_{u}$, and arm partial rankings $\mathcal{P}_{a}$\\
  \textbf{Initialize:} Active copies of all rankings; $H_j\gets\emptyset$ for all $j\in[K]$\\
  \Repeat{no pair was deleted in this iteration or some user has an empty active list}{
    Set $m_i\gets\emptyset$ for all $i\in[N]$\\
    \For{each user $i \in [N]$}{
        Propose to all `source' arms in the active partial rank $P_{u,i}$
    }
    \For{arm $j \in [K]$ with a nonempty proposal history or a new proposal}{
        Gather the proposing users $Prop(j)$; set $H_j\gets H_j\cup Prop(j)$\\
        Get the `source' nodes $S(j)$ in the frozen arm order restricted to $H_j$\\
        \For{each active user $k <_{P_{a,j}}s$, for some $s\in S(j)$}{
            Remove arm $j$ from $P_{u,k}$, and user $k$ from the active $P_{a,j}$
        }
        \If{$|S(j)|=1$}{
            For the unique $i^*_j\in S(j)$, append $j$ to $m_{i^*_j}$
        }
        \Else{
            \For{each $k\in S(j)$ whose pair $(k,j)$ is still active}{
                Remove arm $j$ from $P_{u,k}$, and user $k$ from the active $P_{a,j}$
            }
        }
    }
  }
  \If{every $m_i$ is a singleton, no user has an empty active list, and every arm with $H_j\ne\emptyset$ holds a proposal}{
    Return $M_{stable}=\{(i,j):m_i=\{j\},\ i\in[N]\}$
  }
  Return $\emptyset$
\end{algorithm}
The test is made after an iteration with no deletions; an empty user list instead gives immediate failure. The maximal historical proposers are rejected when there is more than one, rather than deleting arbitrary sink vertices. Rejected proposals remain in $H_j$. Thus a later holder must be strictly preferred to every distinct earlier proposer. Lemma~\ref{lemm:oracle_correctness} proves correctness and the phase-length bound for this procedure.

Next, we state the full decentralized algorithms here for the arms and the users. The frozen phase ranks are the estimates available before its first reward round. Each arm reports newly deleted pairs to their users through the local reply channels, including pairs that have not yet proposed. Both endpoints then remove those pairs. Proposal sets and current acceptance records are reset each round; proposal histories persist until the end of the phase. At a phase boundary, all users and arms read the preceding outcome and save any certified partners before either shared flag is reset. These replies implement the same simultaneous iteration as Algorithm~\ref{alg:extended-GS}. 

\begin{algorithm}[!htbp]
\caption{Decentralized Two-Sided Matching Bandit, Arm $j$}
\label{alg:decentralizedArm}
\textbf{Global Flags:} $restart$, $success$\\
\textbf{Counter:} Common exploration counter $\tau_{ex}\gets0$; phase counter $\tau_{\ell} \gets 1$\\
\textbf{Initial Ranking:} $P_{a, j}(1) \gets \text{empty $N$ node DAG}$\\
$\triangleright$ Let global time $t \gets 0$ (convention)\\
\For{$\tau_{\ell} \geq 1$}{
Read the preceding $success$ value and set $explore(\tau_\ell)\gets\lnot success$; if successful, save the unique user accepted at the end of the preceding phase (or $\emptyset$ if this arm held none); freeze $P_{a,j}(t+1)$ and initialize its active copy $PR_{a,j}(\tau_\ell)$ and $H_j\gets\emptyset$  \\
\While{$True$}{
    $\triangleright$ Increment global time $t \gets t + 1$ \\
    Receive proposals from users $\tilde{S}_{j}(t)$\\
    Set $H_j\gets H_j\cup\tilde S_j(t)$; compute its maximal set $I^*(j)$ in the frozen arm order\\
    Delete all active users strictly dominated by some member of $I^*(j)$\\
    \uIf{$|I^*(j)|=1$}{
        Accept its unique member $i^*(j)$; reject all other current proposers
    }
    \Else{
        Reject all current proposers and delete every still-active member of $I^*(j)$
    }
    Send a rejection notification to every user whose pair was deleted, and report current acceptance status to proposers\\
    $\triangleright$ Arm $j$ matches with user $m_j^{-1}(t)$\\ 
    $\triangleright $$\, m^{-1}_j(t) = \{i:1+((i-1+\tau_{ex})\bmod K)=j\}$ (possibly $\emptyset$) if $explore(\tau_{\ell}) = True$\\
    $\triangleright \, m^{-1}_j(t)$ is the saved partner from the preceding successful phase if $explore(\tau_{\ell}) = False$ \\
    \If{$m^{-1}_j(t) \neq \emptyset$}{
         Observe reward $\tilde{Y}_{j,m^{-1}_j(t)}(t)$ and update the corresponding empirical mean and count\\
    }
    If exploring, increment the local copy of $\tau_{ex}$ after physical play\\
    Recompute all confidence bounds at $t+1$ and $P_{a, j}(t+1)$ using Eqs.~\eqref{eq:arm-cb}--\eqref{eq:pr}.\\ 
    \vspace{0.5em}$\triangleright$ GLOBAL COMMUNICATION \\
    Read shared flag $restart$ (once all users finish update)\\
    \If{$restart$}{
       \text{break} \quad$\triangleright$  Enters a new phase\\
    }
}
}
\end{algorithm}

\begin{algorithm}[!htbp]
\caption{Decentralized Two-Sided Matching Bandit, User $i$}
\label{alg:decentralizedUser}
\textbf{Global Flags:} $restart \gets False$, $success \gets False$ \\
\textbf{Local Indices:} Exploration $\tau_{ex} \gets 0$, Phase count $\tau_{\ell} \gets 1$, \\
\hspace{6.3em} Phase exploration indicator $explore(1) \gets True$\\
\textbf{Initial Ranking:} $P_{u, i}(1) \gets  \text{empty $K$ node DAG}$\\
$\triangleright$ Let global time $t \gets 0$ (convention)\\
\For{$\tau_{\ell} \geq 1$}{
Update $PR_{u,i}(\tau_\ell)\gets P_{u,i}(t+1)$; initialize the phase iteration count and the set of arms proposed to\\
Read shared flag $success$ and set $explore(\tau_{\ell}) \gets \lnot success$\\
\If{$success$}{
    $M_i(\tau_\ell)\gets$ the unique certified partner from the preceding phase
}
\While{$True$}{
     $\triangleright$ Reset $restart\gets False$, $success\gets False$ in public order\\
     $\triangleright$ Increment global time $t \gets t + 1$ \\
     Propose for all arms $j$ in `source' of  $PR_{u, i}(\tau_{\ell})$, $\tilde{S}_j(t) \gets  \tilde{S}_j(t) \cup \{i\}$ \\
     Receive acceptance statuses and all deletion notifications from the arms $S_{i, j}(t)$ \\
    Record proposed arms and current acceptances; delete every arm sending a rejection notification from $PR_{u, i}(\tau_{\ell})$ \\
    \If{$explore(\tau_{\ell})$}{
         Match with $m_i(t) \gets 1+((i-1+\tau_{ex})\bmod K)$, and
        set $\tau_{ex} \gets  \tau_{ex} + 1$.      
    }
    \Else{
         Match with $m_i(t) \gets M_i(\tau_{\ell})$. \quad$\triangleright$ Match from previous phase
    }
    $\triangleright$ No collision occurs, and $m_i(t) \neq \emptyset$\\
    Observe reward $Y_{i,m_i(t)}(t)$ and update the corresponding empirical mean and count. \\
    Recompute all confidence bounds at $t+1$ and $P_{u, i}(t+1)$ using Eqs.~\eqref{eq:user-cb}--\eqref{eq:pr}.\\
    \vspace{0.5em}$\triangleright$ GLOBAL COMMUNICATION \\
    Increment the phase iteration count; execute the phase-control reductions described in the phase-control paragraph; their outputs are $success$ and $restart$\\
    Read shared flag $restart$ (once all updates are done)\\
    \If{$restart$}{
        \text{break} $\triangleright$  Enters a new phase\\
    }
}
}
\end{algorithm}

\paragraph{Phase control.}
After the local replies and the reward observation, users compute, by ordered reductions on the two shared flags, whether any pair was deleted in this iteration and whether any user has an empty active list. An empty list ends the phase with $success=False$ and $restart=True$. Otherwise, if a deletion occurred, the phase continues, except that it is ended after $H=N^2+1$ iterations with $success=False$ and $restart=True$. If there was no deletion, the users perform the terminal test. First check that every user has exactly one current acceptance. For each arm $j$, in public order, use $restart$ to compute the OR of the users' indicators that they have ever proposed to $j$, and $success$ to compute the OR of their indicators that they are currently accepted by $j$. Reset the flags before each reduction. A designated user retains the singleton-acceptance result locally while the flags are reused, and stores whether every proposed-to arm has a holder. The conjunction of these two results is the final $success$ value; set $restart=True$ to end the phase. Both checks finish in the communication stage of the same physical round. This uses two reusable flags, with multiple accesses, rather than a total of two transmitted bits per round. The cap does not truncate an admissible phase, by Lemma~\ref{lemm:oracle_correctness}.

\section{Proofs of Preliminary Section}
\begin{proof}[Proof of Lemma~\ref{lemm:structural-main}]
Given $(F_u, F_a) \in \mathrm{FullRank}(P_u, P_a)$, we need to show that $\mathrm{SuperStable}(P_{u}, P_{a}) \cap \mathrm{Stable}(F_u, F_a) \neq \emptyset$ in order to prove this lemma. It is sufficient to argue that $\mathrm{FullRank}(P_u, P_a) \subseteq \mathrm{FullRank}(P'_u, P'_a)$   for {\em some} $(P'_u, P'_a) \in \mathcal{A}(F_u, F_a)$, because 
\begin{align*}
 \mathrm{Stable}(F_u, F_a)  \cap \mathrm{SuperStable}(P_{u}, P_{a})&
 \overset{(i)}{=} \mathrm{Stable}(F_u, F_a) \cap  \underset{(F'_{u},F'_{a}) \in \mathrm{FullRank}(P_{u}, P_{a})}{\bigcap} \mathrm{Stable}(F'_{u}, F'_{a}) \\
 &\overset{(ii)}{\supseteq} \mathrm{Stable}(F_u, F_a) \cap    \underset{(F'_{u},F'_{a}) \in \mathrm{FullRank}(P'_{u}, P'_{a})}{\bigcap} \mathrm{Stable}(F'_{u}, F'_{a}) \\
 &\overset{(iii)}{=} \mathrm{Stable}(F_u, F_a) \cap \mathrm{SuperStable}(P'_{u}, P'_{a})
 \overset{(iv)}{\neq} \emptyset 
\end{align*}
The equality $(i)$ and $(iii)$ is by the definition of super-stable set. Inequality $(ii)$ follows from the assumption $\mathrm{FullRank}(P_u, P_a) \subseteq \mathrm{FullRank}(P'_u, P'_a)$, and inequality $(iv)$ holds because we chose $(P'_u, P'_a)  \in \mathcal{A}(F_u, F_a)$.

What is left to complete the proof is to show that the assumption $\mathrm{FullRank}(P_u, P_a) \subseteq \mathrm{FullRank}(P'_u, P'_a)$ is true for some $(P'_u, P'_a) \in \mathcal{A}(F_u, F_a)$. Let us choose a partial rank $(P'_u, P'_a) \in \mathcal{A}(F_u, F_a)$ with a minimum gap $\Delta_{\mathcal{A}}(\boldsymbol{\mu}, \boldsymbol{\gamma})$ which exists due to the finiteness of the set $\mathcal{A}(F_u, F_a)$. For distinct users $i,i'\in[N]$ and distinct arms $j,j'\in[K]$, the respective comparisons satisfy, by definition of overlap width, and because 
$\mathcal{W}_{\mathrm{ov}}(P_u, P_a; \boldsymbol{\mu}, \boldsymbol{\gamma}) < \Delta_{\mathcal{A}}(\boldsymbol{\mu}, \boldsymbol{\gamma})$
\begin{align*}
    &|\mu_{i,j} - \mu_{i, j'}| \geq  \Delta_{\mathcal{A}}(\boldsymbol{\mu}, \boldsymbol{\gamma}) \implies j \underset{P_{u,i}}{\not\parallel} j' \text{ and }
    |\gamma_{j,i} - \gamma_{j,i'}| \geq  \Delta_{\mathcal{A}}(\boldsymbol{\mu}, \boldsymbol{\gamma}) \implies i \underset{P_{a,j}}{\not\parallel} i'.
\end{align*}
However, by definition of $\Delta_{\mathcal{A}}(\boldsymbol{\mu}, \boldsymbol{\gamma})$ and the fact that $(P'_u, P'_a)$ attains the maximum gap we know
\begin{align*}
    &j \underset{P'_{u,i}}{\not\parallel} j' \implies |\mu_{i,j} - \mu_{i, j'}| \geq  \Delta_{\mathcal{A}}(\boldsymbol{\mu}, \boldsymbol{\gamma}) \implies j \underset{P_{u,i}}{\not\parallel} j' \\
     &i \underset{P'_{a,j}}{\not\parallel} i' \implies  |\gamma_{j,i} - \gamma_{j,i'}| \geq  \Delta_{\mathcal{A}}(\boldsymbol{\mu}, \boldsymbol{\gamma}) \implies i \underset{P_{a,j}}{\not\parallel} i'.
\end{align*}
In other words, any pair of users (or arms) that are comparable in the partial rank $(P'_u, P'_a)$ are also comparable in the partial rank $(P_u, P_a)$. Both ranks are compatible with the same true full rank, so each such comparison has the same direction. This in turn implies that $\mathrm{FullRank}(P_u, P_a) \subseteq \mathrm{FullRank}(P'_u, P'_a)$. This completes our proof.
\end{proof}

\section{Proofs of Regret Upper Bounds}

For each reward stream, let $\widehat\mu_{ij}^{(s)}$ and $\widehat\gamma_{ji}^{(s)}$ denote the empirical means of its first $s$ observations. For each round $t\geq2$, define the good event
\[
\mathcal G_t=\bigcap_{i\in[N],j\in[K]}\bigcap_{s=1}^{t-1}
\left\{\max\{|\widehat\mu_{ij}^{(s)}-\mu_{ij}|,|\widehat\gamma_{ji}^{(s)}-\gamma_{ji}|\}
\leq\sqrt{6\log(t)/s}\right\}.
\]
Set $\mathcal G_1=\Omega$, since all initial confidence intervals are unbounded.
\begin{lemma}\label{lemm:bad-event-prob}
For $t\geq2$, $\mathbb P(\mathcal G_t^c)\leq4NK(t-1)t^{-3}\leq4NK/t^2$. Consequently,
$\sum_{t=1}^{\infty}\mathbb P(\mathcal G_t^c)\leq2\pi^2 NK/3$.
\end{lemma}
\begin{proof}
For a fixed stream and deterministic sample count $s$, the two-sided $1$-subgaussian concentration bound is
\[
\mathbb P\{|\widehat m^{(s)}-m|>\sqrt{6\log(t)/s}\}
\leq2\exp(-3\log(t))=2t^{-3}.
\]
A union bound over $2NK$ streams and $t-1$ sample counts gives the first inequality; summing $t^{-2}$ gives the stated bound. The union over deterministic counts makes the event valid under adaptive sampling, without conditioning on a random observation count.
\end{proof}

\begin{lemma}\label{lemm:contains-true}
    Conditioned on the good event $\mathcal G_t$, the recovered set of partial matching $(\mathcal{P}_{u}(t), \mathcal{P}_{a}(t))$ satisfies that  $(F^*_u, F^*_a) \in FullRank((\mathcal{P}_{u}(t), \mathcal{P}_{a}(t)))$.  
\end{lemma}
\begin{proof}
    We know that given the good event $\mathcal G_t$, for each observed pair $(i,j)$ we have 
    \begin{align*}
        & \uucb_{i, j}(t) \geq \mu_{i,j} +  \sqrt{\tfrac{6\log(t)}{n_{i,j}(t)}}  - |\hat{\mu}_{i,j}(t) - \mu_{i,j}| \geq \mu_{i,j}, \\
        & \ulcb_{i, j}(t) \leq \mu_{i,j} -  \sqrt{\tfrac{6\log(t)}{n_{i,j}(t)}}  + |\hat{\mu}_{i,j}(t) - \mu_{i,j}| \leq \mu_{i,j}.
    \end{align*}
    Unobserved pairs have unbounded confidence intervals and yield no strict comparisons. Therefore, $\uucb_{i, j'}(t) < \ulcb_{i, j}(t)$ implies that $\mu_{i,j} > \mu_{i, j'}$ under the good event $\mathcal G_t$. Similarly,  $\aucb_{j,i'}(t) < \alcb_{j,i}(t)$ implies $\gamma_{j,i} > \gamma_{j,i'}$ under the good event $\mathcal G_t$.

    Thus every strict comparison in the recovered partial rank agrees with the true full rank, on both sides of the market. This is exactly the condition that $(F_u^*,F_a^*)\in\mathrm{FullRank}(\mathcal P_u(t),\mathcal P_a(t))$.
\end{proof}

\begin{lemma}[Convergence to Admissible Set]\label{lemm:admissible}
Conditioned on the good event $\mathcal G_t$, for any round $t\geq1$ if number of explorations $N_{explore}(t) > K + \frac{96K\log(t)}{\Delta^2_{\mathcal{A}}(\boldsymbol{\mu}, \boldsymbol{\gamma})}$ we have $(\mathcal{P}_{u}(t), \mathcal{P}_{a}(t)) \in \mathcal{A}(F^*_{u}, F^*_{a})$, i.e. the recovered partial rank lies in the set of admissible partial rank instances of the true rankings.
\end{lemma}
\begin{proof}[Proof of Lemma~\ref{lemm:admissible}]
   Let $N_{explore}(t)$ be the number of exploration steps taken by the centralized platform strictly before time $t$. That implies that for each user and arm pair $(i, j)$ for $i \in [N]$ and $j\in [K]$, the number of samples $n_{i,j}(t) \geq (N_{explore}(t) - K) / K$.  With $n_{i,j}(t)$ many samples, under the good event $\mathcal G_t$ by definition
   \begin{align*}
       \max\Big(|\mu_{i,j} - \hat{\mu}_{i,j}(t)|, |\gamma_{j,i} - \hat{\gamma}_{j,i}(t)| \Big) \leq \sqrt{\frac{6\log(t)}{n_{i,j}(t)}}.
   \end{align*}
    Therefore, for any fixed user $i \in [N]$ and any pair of $j, j' \in [K]$, if 
    $$
    (\mu_{i,j} - \mu_{i,j'}) > 4\sqrt{\frac{6K\log(t)}{N_{explore}(t) - K}} \geq  2\sqrt{6\log(t)}\big(\tfrac{1}{\sqrt{n_{i,j}(t)}} + \tfrac{1}{\sqrt{n_{i,j'}(t)}}\big),
    $$ then we have $\ulcb_{i,j}(t) > \uucb_{i,j'}(t)$ and $j \underset{\mathcal{P}_{u, i}(t)}{>} j'$. Similarly, if $(\gamma_{j,i} - \gamma_{j, i'}) > 4\sqrt{\tfrac{6K\log(t)}{N_{explore}(t) - K}}$ we have $i \underset{\mathcal{P}_{a, j}(t)}{>} i'$. Therefore, by definition $\mathcal{W}_{\mathrm{ov}}(\mathcal{P}_u(t), \mathcal{P}_a(t); \boldsymbol{\mu}, \boldsymbol{\gamma}) \leq  4\sqrt{\tfrac{6K\log(t)}{N_{explore}(t) - K}}$.
    
    From Lemma~\ref{lemm:contains-true} $ (F^*_{u}, F^*_{a}) \in \mathrm{FullRank}(\mathcal{P}_{u}(t), \mathcal{P}_{a}(t))$ under the good event $\mathcal G_t$. Therefore, under good event $\mathcal G_t$, from Lemma~\ref{lemm:structural-main} if 
    $\mathcal{W}_{\mathrm{ov}}(\mathcal{P}_u(t), \mathcal{P}_a(t); \boldsymbol{\mu}, \boldsymbol{\gamma}) < \Delta_{\mathcal{A}}(\boldsymbol{\mu}, \boldsymbol{\gamma})$ then $(\mathcal{P}_{u}(t), \mathcal{P}_{a}(t)) \in \mathcal{A}(F^*_{u}, F^*_{a})$. Therefore, for all $t$ such that  
    $$
    N_{explore}(t) > K + 96 K \frac{\log(t)}{\Delta^2_{\mathcal{A}}(\boldsymbol{\mu}, \boldsymbol{\gamma})} 
    $$  
    we have $(\mathcal{P}_{u}(t), \mathcal{P}_{a}(t)) \in \mathcal{A}(F^*_{u}, F^*_{a})$.
\end{proof}

\begin{lemma}\label{lemm:explore}
    Conditioned on the good event $\mathcal G_t$, for each round $t\geq1$, $(\mathcal{P}_{u}(t), \mathcal{P}_{a}(t)) \notin \mathcal{A}(F^*_{u}, F^*_{a})$ if and only if exploration is triggered in round $t$.
\end{lemma}
\begin{proof}
    For round $t\geq 1$, if Algorithm~\ref{alg:extended-GS} does not return a stable match then exploration is triggered in  Algorithm~\ref{alg:centralized}. Therefore, we need to show that under $\mathcal G_t$, if  $(\mathcal{P}_{u}(t), \mathcal{P}_{a}(t)) \notin \mathcal{A}(F^*_{u}, F^*_{a})$ then Algorithm~\ref{alg:extended-GS} does not return a stable match. Note that from Lemma~\ref{lemm:contains-true} we have $(F^*_{u}, F^*_{a}) \in \mathrm{FullRank}(\mathcal{P}_{u}(t), \mathcal{P}_{a}(t))$ under the good event $\mathcal G_t$. If there is a stable match returned by Algorithm~\ref{alg:extended-GS}  in round $t$, then $\mathrm{SuperStable}(\mathcal{P}_{u}(t), \mathcal{P}_{a}(t)) \neq \emptyset$. Therefore, by definition $(\mathcal{P}_{u}(t), \mathcal{P}_{a}(t)) \in \mathcal{A}(F^*_{u}, F^*_{a})$. Conversely, if the recovered rank is admissible, Lemma~\ref{lemm:oracle_correctness} returns a super-stable matching and exploration is not triggered. This proves both directions.
\end{proof}

\begin{proof}[\textbf{Proof of Centralized Upper Bound Theorem}~\ref{thm:centralized_main}]
Fix an analysis horizon $T\geq1$, which is not an input to the algorithm, and let
\[
q_T=\left\lfloor\frac{96\log(T)}{\Delta_{\mathcal A}^2(\boldsymbol\mu,\boldsymbol\gamma)}\right\rfloor+1,
\qquad Th(T)=Kq_T.
\]
After $Th(T)$ exploratory rounds, every pair has at least $q_T$ samples, because each $K$ consecutive round-robin assignments samples every pair once. For any later round $t\leq T$, on $\mathcal G_t$ the overlap-width argument in Lemma~\ref{lemm:admissible} gives
\[
\mathcal W_{\rm ov}(\mathcal P_u(t),\mathcal P_a(t);\boldsymbol\mu,\boldsymbol\gamma)
\leq4\sqrt{6\log(T)/q_T}<\Delta_{\mathcal A}.
\]
Thus the recovered rank is admissible and exploration is not triggered in such a round. Before this threshold, at most $Kq_T$ exploratory rounds occur. By Lemma~\ref{lemm:oracle_correctness} and Corollary~\ref{corr:super-stable}, every exploitation round on its event $\mathcal G_t$ plays a true stable matching. Hence, pathwise,
\[
R_{0/1}(T)\leq Kq_T+\sum_{t=1}^T\mathbbm1\{\mathcal G_t^c\}.
\]
Taking expectations and applying Lemma~\ref{lemm:bad-event-prob} yields
\[
\mathbb E R_{0/1}(T)\leq K+\frac{96K\log(T)}{\Delta_{\mathcal A}^2}+\frac{2\pi^2 NK}{3}.
\]
A stable matching gives every agent at least its pessimal stable reward. On any other round, the positive difference from this benchmark is at most one, since all means and the unmatched reward lie in $[0,1]$. This proves the user and arm regret bounds without dividing by a potentially zero benchmark gap.
\end{proof}

\begin{lemma}[Extended-GS correctness and phase length]\label{lemm:oracle_correctness}
For fixed input partial orders on $N\leq K$ agents with all pairs acceptable, Algorithm~\ref{alg:extended-GS} returns a super-stable matching if one exists and returns $\emptyset$ otherwise. It terminates within $NK+1$ iterations. If its input admits a super-stable matching, it terminates within $N^2+1$ iterations.
\end{lemma}
\begin{proof}
The rankings are frozen throughout an execution. A unique maximal element of a finite partial order strictly dominates every other element. Thus an arm with a unique maximal historical proposer holds that proposer and deletes every strictly dominated active pair. With multiple maximal historical proposers, every historical proposal is deleted: nonmaximal ones are strictly dominated, and maximal ones are deleted by the second rule. Once a historical proposal has been deleted, it cannot later become the unique maximal historical proposal, since its dominating or incomparable historical witness is retained. In particular, every engagement recorded by the algorithm is active.

Every super-stable matching covers all $N$ users: otherwise an unmatched user and an unmatched arm would block, since $N\leq K$ and all pairs are acceptable. Fix a super-stable matching $M$ of the input, and suppose $(i,j)\in M$ is the first edge of $M$ deleted. If the deletion is by strict domination, there is a historical proposer $v$ with $v>_{P_{a,j}}i$. If it is a maximal-proposer deletion, there is another maximal historical proposer $v$ incomparable with $i$. At the time $v$ proposed, $M(v)$ was still active. Because $j$ was maximal in $v$'s active list, $M(v)\not>_{P_{u,v}}j$. Also $i\not>_{P_{a,j}}v$ in both cases. Hence $(v,j)$ blocks $M$, a contradiction. No edge of $M$ is deleted. Similarly, an arm unmatched in $M$ cannot receive a proposal: its proposer would not strictly prefer its still-active $M$-partner, creating a blocking pair.

Every iteration except the final one deletes a pair or stops on an empty list. There are at most $NK$ pairs. If $M$ exists, all deletions occur at the $N$ arms matched in $M$, so there are at most $N^2$ deletions, no user list becomes empty, and at most $N^2+1$ iterations are required. At a deletion-free iteration, every user has a nonempty maximal set, proposes to all its members, and is accepted by each: any rejected current proposer would have an active pair deleted. There are at least $N$ engagements, with each arm holding at most one. Since only the $N$ arms matched in $M$ have been visited, each user has exactly one engagement and all visited arms are engaged. The terminal test therefore succeeds whenever a super-stable matching exists.

Conversely, suppose the terminal test returns $E$. Consider $(i,j)\notin E$. If the pair remains active, the deletion-free iteration implies that every maximal active arm of $i$ accepted $i$. As $i$ has exactly one engagement, $E(i)$ is its unique maximal active arm, and $E(i)>_{P_{u,i}}j$. If $(i,j)$ was deleted by strict domination, some historical proposer $s$ satisfies $s>_{P_{a,j}}i$. The final holder of $j$ is the unique maximal historical proposer, and therefore is $s$ or strictly dominates $s$; in either case it strictly dominates $i$. If the pair was deleted as one of multiple maximal historical proposals, $i$ belongs to $H_j$ and is distinct from the final unique maximal proposer, which again strictly dominates $i$. Such a holder exists by the terminal test. Thus every outside pair is strictly rejected by at least one endpoint, proving super-stability.
\end{proof}

\begin{proof}[\textbf{Proof of Decentralized Upper Bound Theorem}~\ref{thm:decentralized_main}]
The common flags synchronize phase boundaries, and the local proposal/reply rules implement exactly one iteration of Algorithm~\ref{alg:extended-GS} in each physical round. In particular, deleted pairs are reported to both endpoints, and all historical comparisons use the frozen phase ranks. The phase-control reductions compute the same empty-list, deletion, and terminal-matching predicates as the centralized procedure. By Lemma~\ref{lemm:oracle_correctness}, the cap $H=N^2+1$ never truncates a phase whose input is admissible.

Let $s_\ell$ be the first reward round of phase $\ell$. Its frozen rank uses the observations strictly before $s_\ell$. Call this phase bad if $\mathcal G_{s_\ell}^c$ occurs. A successful phase that is not bad certifies a true stable matching, which is played throughout the following phase. Every phase lasts at most $H=N^2+1$ rounds, including phases with incorrect or inadmissible input.

Let $q_T$ be as in the centralized proof and set $L=Kq_T$. Up to the $L$-th exploratory round, exploratory regret is at most $L$. If that round occurs, also allow $2H$ rounds for the remainder of its phase and the immediately following phase. Every subsequent exploratory phase follows a failed phase whose start already had at least $L$ exploratory observations. Such a predecessor must be bad: otherwise its frozen rank is admissible, by the same overlap-width bound as in the centralized proof, and the oracle succeeds within $H$ rounds. Charge the at most $H$ rounds of this exploratory phase to that bad predecessor. If the exploration count never reaches $L$, all exploratory rounds are already covered by the initial allowance.

An exploitation phase can play an unstable matching only if the preceding successful phase was bad. Charge its at most $H$ rounds to that predecessor as well. Each bad phase has just one successor, which is exploratory after failure and exploitative after success, so it receives at most one charge of $H$. Therefore, pathwise,
\[
R_{0/1}(T)\leq Kq_T+2H+H\sum_{\ell:s_\ell\leq T}\mathbbm1\{\mathcal G_{s_\ell}^c\}
\leq Kq_T+2H+H\sum_{t=1}^T\mathbbm1\{\mathcal G_t^c\}.
\]
The last inequality holds for every realization, so no concentration bound is conditioned on a random phase start. Taking expectations gives
\[
\mathbb E R_{0/1}(T)\leq K+\frac{96K\log(T)}{\Delta_{\mathcal A}^2}
+2(N^2+1)+\frac{2\pi^2 NK(N^2+1)}{3}.
\]
The user and arm inequalities follow from the same one-round comparison as in the centralized proof.
\end{proof}

\section{Proof of Regret Lower Bounds}\label{app:lower}
In this section, we present proof of the theorems and lemmas related to the lower bound result. Our lower bound follows a framework similar to Combes et al.~\cite{combes2017minimal}, which in turn relies on Graves et al.~\cite{graves1997asymptotically}.  However, Combes et al.~\cite{combes2017minimal} does not capture our use case as it is tailored to bandit feedback. due to the multiplicity of the stable matching set there is no easy way to encode this problem in the form of linear Semi-bandit like Combes et al.~\cite{combes2015combinatorial}. 

We recall our system definition. Our system is parameterized by $\theta = (\boldsymbol{\mu}, \boldsymbol{\gamma}) \in  \Theta\ \text{as defined in Section~\ref{sec:lower}}$. We extend our notations (e.g. $\mathrm{Stable}(F_u, F_a)$, $\mathcal{A}((\boldsymbol{\mu}, \boldsymbol{\gamma}))$)  to replace both $(\boldsymbol{\mu}, \boldsymbol{\gamma})$ and $(F_u, F_a)$  with $\theta$, as $\theta$ uniquely specifies our instance.  The action space is the space of all matching between bipartite graphs between two sides of size $N$ and $K$, denoted as $\mathrm{Match}(N,K)$. In round $t$, choosing $M(t)$ yields independent observations $(Y_i(t),\tilde Y_j(t))$ for each matched pair $(i,j)\in M(t)$, with Bernoulli means $\mu_{ij}(\theta)$ and $\gamma_{ji}(\theta)$. Unmatched pairs yield no observation. The one-action KL divergence is therefore the sum of the endpoint divergences over the pairs in $M(t)$. For the sake of simplicity, for this lower-bound analysis we work with Bernoulli variables. Extending it to a general single-parameter distribution can be done similar to Lai et al.~\cite{lai_1985}.  

Recall the no-common-solution alternatives $\Lambda(\theta)$ from Section~\ref{sec:lower}, and the admissible-rank alternatives $\Lambda_{\mathcal A}(\theta)$ from Equation~\eqref{eq:bad_set}. Their sets need not agree, but Lemma~\ref{lemm:alternative_domination} shows that their information constraints agree. We prove the following result directly, including the multiple-solution case, in Section~\ref{sec:proof_lower_bound}.
\begin{theorem}[Information lower bound]\label{thm:centralized_lower_combes}
For all $\theta \in \Theta$, for any {\em uniformly good} policy $\pi$
$\lim\inf_{T\to \infty} \frac{\mathbb E[R_{0/1}^{\pi}(T;\theta)]}{\log(T)} \geq c(\theta)$, where $c(\theta)$ minimizes the following optimization problem 
\begin{align}
    \min_{\eta(M) \geq 0} \notag
    &\sum_{M \in \mathrm{Match}(N,K)\setminus \mathrm{Stable}(\theta)}\eta(M)\\
    \text{subject to} &\sum_{(i, j)}\sum_{M: (i,j) \in M}\eta(M) kl(\theta, \lambda; (i,j))  \geq 1, \forall \lambda \in \Lambda_{\mathcal A}(\theta),\label{eq:primal}
\end{align}
where $\eta(M)=0$ for stable matchings, without loss of generality, and $\Lambda_{\mathcal A}(\theta)$ is defined in Equation~\eqref{eq:bad_set}. 
\end{theorem}  

\subsection{Instantiation of Lower Bound for Special Cases}
We now construct some regret lower bounds using Theorem~\ref{thm:centralized_lower_combes}.

\textbf{Vanishing information lower bound:} When every edge is locked, preserving the locked means forces $\lambda=\theta$, so $\Lambda(\theta)=\emptyset$ and $c(\theta)=0$. A uniformly good policy with bounded regret at a fixed such target instance is constructed in the proof of Corollary~\ref{corr:wide_sub_log}; the construction includes validation and does not assume that the target is the true instance.

\paragraph{General Serial Dictatorship with Two sided Uncertainty:}
In serial dictatorship all the arms share the same preference. With out loss of generality (up to renaming users) let us assume  user $i$ is preferred than user $i'$ for $i < i'$. There exists a unique stable match which again without loss of generality (up to renaming arms) we can denote as $(i,i)$ for all $i \in [N]$. Choose $\varepsilon>0$ small enough that each changed mean stays in $(0,1)$ and all within-agent preferences remain strict; limits below are along such choices. Then we have dominated arm for user $i$ as $D_i(\theta)=\{j<i:\mu_{i,j}(\theta)>\mu_{i,i}(\theta)\}$ for all $i \leq N$. Dominated arms are the armed preferred by user $i$ compared to it's own stable matched arm. We can construct an instance $\lambda_{\gamma,(ij)}$ by picking a $j \in D_i(\theta)$ and then making user $i$ preferable to  user $j$ for arm $j$, i.e. new $\gamma'_{j,i} = \gamma_{j,j}(\theta) + \varepsilon$. We have $kl(\theta, \lambda_{\gamma,(ij)}) = kl(\gamma_{j,i}(\theta), \gamma_{j,j}(\theta) + \varepsilon)$. Because serial dictatorship has unique stable match, we now have a blocking (user, arm) pair $(i,j)$ in $\lambda_{\gamma,(ij)}$ for the matching $\{(i,i): i \in [N]\}$, and therefore $\lambda_{\gamma,(ij)} \in \Lambda(\theta)$. We have $\Lambda_1 = \{\lambda_{\gamma,(ij)}: i \in [N], j\in D_i(\theta)\} \subseteq \Lambda(\theta)$. Next, we consider change on the user side. Let $\lambda_{\mu,(ij)}$ be constructed by picking any $j > i$, and making new $\mu'_{i,j} = \mu_{i,i}(\theta) + \varepsilon$. This way now $(i,j)$ again create blocking pairs and as a consequence we have $\lambda_{\mu,(ij)} \in \Lambda(\theta)$. Let  $\Lambda_2 = \{\lambda_{\mu,(ij)}: i \in [N], j > i\} \subseteq \Lambda(\theta)$. 
Then the optimization problem \eqref{eq:primal}, relaxed by replacing $\Lambda(\theta)$ with $\Lambda_1\cup \Lambda_2$ has the instantiation 
\begin{align*}
&\min_{\eta(M) \geq 0} \sum_{M\neq \{(i,i): i\in [N]\}} \eta(M)\\
&\text{ s.t. }\sum_{M\ni(i,j)}\eta(M) \geq w_{i,j}, \forall (i,j),\\
&w_{i,j} kl(\gamma_{j,i}(\theta), \gamma_{j,j}(\theta) + \varepsilon) \geq 1, \forall i \in [N], \forall j \in D_i(\theta)\\
&w_{i,j} kl(\mu_{i,j}(\theta), \mu_{i,i}(\theta) + \varepsilon) \geq 1, \forall i \in [N], \forall j > i
\end{align*}
These alternatives belong to the full model $\Theta$ and need not themselves be serial dictatorship instances; uniform goodness is required on $\Theta$. Letting $\varepsilon\downarrow0$ in each one-edge constraint, and using continuity of Bernoulli KL, gives the following bounds. Because each user $i$ can match with only one arm $j$ in each matching we have a lower bound for the above optimization problem as 
$$
c_{u,i} = \sum_{j\in D_i(\theta)}kl(\gamma_{j,i}(\theta), \gamma_{j,j}(\theta))^{-1} + \sum_{j\in [i+1, K]} kl(\mu_{i,j}(\theta), \mu_{i,i}(\theta))^{-1}.
$$
Similarly, considering the perspective of each arm $j$ we obtain a lower bound 
$$c_{a,j}(\theta) =\sum_{i: j \in D_i(\theta)}kl(\gamma_{j,i}(\theta), \gamma_{j,j}(\theta))^{-1} + \sum_{i< j} kl(\mu_{i,j}(\theta), \mu_{i,i}(\theta))^{-1}.
$$
Thus the final lower bound for the general serial dictatorship for binary stable regret is 
$$
c(\theta) \geq \max\left(\max_{i \in [N]}c_{u,i}(\theta), \max_{j \in [K]}c_{a,j}(\theta)\right).
$$

This lower bound is different than the lower bound in \cite{ucbd3} in multiple ways. Firstly, our bound is for binary stable regret, instead of per user individual regret like the previous one. Further note the lower bound therein does not apply for general serial matching instance, but for further specialized OSB instances. Also the bound in the previous work only uses user sided uncertainty, not both side uncertainty. So these bounds are incomparable. 

\begin{proof}[Proof of Corollary~\ref{cor:redundant_lower}]
Let $J_0(\theta)=\{j:(i,j)\notin\mathrm{Locked}(\theta)\text{ for every }i\}$. An arm $j\in J_0(\theta)$ is unmatched in every $\theta$-stable matching. Because every real pair is acceptable, user $i$ must strictly prefer its partner in any such matching to $j$; otherwise $(i,j)$ would block. The maximum $\mu_{i,j^*(i)}(\theta)$ ranges over all of $i$'s partners in true stable matchings and is strictly below $1$. Raise only the unlocked mean $\mu_{ij}$ to $\mu_{i,j^*(i)}(\theta)+\delta$ for a small generic $\delta>0$. The pair $(i,j)$ then blocks every original stable matching, and the alternative preserves all locked means. Passing to the limit in its information constraint yields
\[
x_{ij}\geq
\frac1{kl(\mu_{ij}(\theta),\mu_{i,j^*(i)}(\theta))}
\qquad(i\in[N],\ j\in J_0(\theta)).
\]
Here $x_{ij}=\sum_{M\ni(i,j)}\eta(M)$ over unstable actions. As above, a matching observes at most one tested pair per user row or arm column. The largest resulting row or column requirement is the asserted bound.
\end{proof}

\subsection{Regret Lower Bound for General Instances}
In this part, we connect the lower bound in Theorem~\ref{thm:centralized_lower_combes} with the comparisons defining the admissible set of $\theta$.
Write $J_i(\theta)=\{j:(i,j)\in\mathrm{Locked}(\theta)\}$ and $I_j(\theta)=\{i:(i,j)\in\mathrm{Locked}(\theta)\}$.
\begin{definition}[$\theta$-Locked] A partial rank $(P_u, P_a)$ is $\theta$-Locked \emph{iff} for all $i \in [N]$ two arms $j$, $j'$ in $J_i(\theta)$ have the order in $P_{u,i}$ identical to $F_{u,i}(\theta)$, and for all $j \in [K]$ two users $i$, $i'$  in $I_j(\theta)$ have the order in $P_{a,j}$ identical to $F_{a,j}(\theta)$.
\end{definition}

For a given instance $\theta$, the boundary of the admissible set, namely $\mathcal{B}(\theta)$, is defined as the set of partial ranks $(P_u, P_a)\in \mathcal{A}(\theta)$ that has no strict coarsening in $\mathcal A(\theta)$, where coarsening means deleting comparisons from the transitive closure. 

\begin{lemma}[$\Lambda_{\mathcal A}(\theta)$ Decomposition]\label{lemm:lambda_decomp}
Given an instance $\theta$ and the boundary admissible set $\mathcal{B}(\theta)$, for each instance $\lambda$ and the corresponding full rank $(F_u(\lambda), F_a(\lambda))$ the following are equivalent:
(a) $\lambda \in \Lambda_{\mathcal A}(\theta)$ and \\
(b) (i) $\lambda$ has exactly the same two reward means as $\theta$ on every locked edge, and (ii) its full rank reverses at least one {\em $\theta$-open triplet} for each $(P_u, P_a) \in \mathcal{B}(\theta)$, as specified below  
\begin{itemize}[leftmargin=*, noitemsep]
    \item a triplet $(i,j, j')$ such that $j \underset{P_{u,i}}{>} j'$, and either $(i,j)$ or $(i, j')$ or both not in $Locked(\theta)$, 
    \item a triplet $(j,i,i')$ such that $i \underset{P_{a,j}}{>} i'$, and either $(i,j)$ or $(i', j)$ or both not in $Locked(\theta)$.
\end{itemize}
\end{lemma}

Lemma~\ref{lemm:lambda_decomp} gives a sufficient condition for the information lower bound to vanish.
\begin{corollary}\label{corr:wide_sub_log}
If there exists a $(\underline{P_u}(\theta), \underline{P_a}(\theta)) \in \mathcal{B}(\theta)$ that contains no $\theta$-open triplet then $\Lambda(\theta) = \emptyset$. Hence, $c(\theta)=0$. Moreover, for each such target $\theta$ there exists a policy uniformly good on $\Theta$ with $O_\theta(1)$ expected binary stable regret at $\theta$.
\end{corollary}

Proof of Lemma~\ref{lemm:lambda_decomp} and Corollary~\ref{corr:wide_sub_log} are deferred to Section\ref{app:additional_proofs}.

More importantly, Lemma~\ref{lemm:lambda_decomp} gives us a way to construct {\em critical} instances by reversing $\theta$-open triplets that form a `cover' of the set $\mathcal{B}(\theta)$ to construct an instance preserving all locked means whose full rank lies outside every admissible completion set, and hence an element of $\Lambda_{\mathcal A}(\theta)$. In fact, it is enough to consider the {\em minimal covers} of the set $\mathcal{B}(\theta)$, which we denote as $Cover(\theta)$.  
User and arm triplets are treated as distinct, tagged objects even when their numerical indices coincide; the tags are suppressed in the notation below.
\begin{definition}[Cover of $\mathcal{B}(\theta)$] \label{def:cover}
A set of triplets $(i,j,j')$ or $(j,i,i')$ for $i,i'\in [N]$ and $j,j'\in [K]$, $C$ is a minimal cover of $\mathcal{B}(\theta)$, i.e. $C \in Cover(\theta)$, \emph{iff} all the following statements hold. Each triplet in $C$ is a true strict comparison that is $\theta$-open in at least one boundary rank.\\
(i) each $(P_u, P_a) \in \mathcal{B}(\theta)$ there exists a triplet $o \in C$ such that $o$ is  $\theta$-open for $(P_u, P_a)$, \\ 
(ii) any proper subset $C' \subset C$, $C'$ is not a cover of $\mathcal{B}(\theta)$,\\
(iii) the set $\Lambda_C$ defined below is non empty
\begin{align}
(\boldsymbol{\tilde{\mu}}^{C}, \boldsymbol{\tilde{\gamma}}^{C}) \in\Theta\text{ belongs to }\Lambda_C\iff &(\boldsymbol{\tilde{\mu}}^{C}, \boldsymbol{\tilde{\gamma}}^{C}) \text{ satisfies} \notag\\
&\tilde{\mu}_{i,j}^{C} =  \mu_{i,j}(\theta), \tilde{\gamma}_{j,i}^{C} =  \gamma_{j,i}(\theta), \quad\forall (i,j) \in Locked(\theta),  \label{eq:lower_4}\\
&\tilde{\mu}_{i,j'}^{C} > \tilde{\mu}_{i,j}^{C},  \quad\forall (i, j, j') \in C,  \label{eq:lower_5}\\
&\tilde{\gamma}_{j,i'}^{C} > \tilde{\gamma}_{j,i}^{C}, \quad\forall (j, i, i') \in C, \label{eq:lower_6}
\end{align}
\end{definition} 

Let $D_C$ be the unlocked user-arm pairs occurring in at least one triplet of $C$. In addition to Equations~\eqref{eq:lower_4}--\eqref{eq:lower_6}, require every reward coordinate outside $D_C$ to equal its value at $\theta$. This restriction does not alter feasibility: coordinates outside the support occur in no selected reversal except as fixed locked coordinates. After fixing them, an arbitrarily small generic perturbation of the unlocked coordinates in $D_C$ removes incidental ties while preserving all strict reversal inequalities.

\begin{definition}[Exclusive Cover Group]\label{def:cover_group}
A set $G\subseteq Cover(\theta)$ is an exclusive cover group if every matching intersects $D_C$ for at most one $C\in G$. The set of such groups is denoted by $\mathcal{CG}(\theta)$. Equivalently, for distinct $C,C'\in G$, the supports $D_C,D_{C'}$ are disjoint and every edge in one shares an endpoint with every edge in the other. Disjointness of $D_C$ and $D_{C'}$ alone is not sufficient.
\end{definition}

The next theorem first gives a feasible dual lower bound using cover alternatives, then a bound obtained by combining exclusive covers. Define
\begin{equation}\label{eq:kl_theta_c}
kl(\theta,C):=\inf_{\lambda\in\Lambda_C}\max_{M\notin\mathrm{Stable}(\theta)}kl(\theta,\lambda;M).
\end{equation}
This variational definition includes covers with comparisons between two unlocked coordinates and enforces all reversals jointly. It is finite and strictly positive for every feasible cover: a strict reversal must change at least one of its two distinct true means by at least half their gap, while any fixed feasible interior alternative has finite KL divergence.

\begin{theorem}\label{thm:centralized_lower_final}
The value of $c(\theta)$ in Theorem~\ref{thm:centralized_lower_combes} is lower bounded by the value $c_{\rm cover}(\theta)$ that optimizes the following
\begin{align}
&\sup_{\iota\geq0,\ |\mathrm{supp}(\iota)|<\infty} \sum_{\lambda \in \cup_{C \in Cover(\theta)}\Lambda_C}\iota(\lambda), \text{ s.t. }\sum_{\lambda \in \cup_{C \in Cover(\theta)}\Lambda_C}\iota(\lambda) kl(\theta, \lambda; M) \leq 1, \forall M \notin Stable(\theta),   \label{eq:lower_1} 
\end{align}
where  $\Lambda_C$ is as defined in Definition~\ref{def:cover}.

Furthermore, $c(\theta)$ satisfies the following lower bound
$c(\theta) \geq \max_{G \in \mathcal{CG}(\theta)}\sum_{C\in G} kl(\theta, C)^{-1}$
where, 
$kl(\theta, C)$ is defined in Equation~\eqref{eq:kl_theta_c}.
\end{theorem}
The proof is provided in Section~\ref{sec:proof_lower_bound}.

Theorem~\ref{thm:centralized_lower_final} specializes the information lower bound of Theorem~\ref{thm:centralized_lower}. 
We now make some remarks on various definitions used. 

\paragraph{Remark on Definition~\ref{def:cover}:} We note in the above definition the condition (i) and (ii) together defines the set of minimal vertex covers of a hyper graph over the triplets as vertices, and the partial ranks $(P_u, P_a)$,  containing their respective subset of $\theta$-open triplets as edges. The condition (iii) ensures that there exists an instance $\lambda$ corresponding to a cover $C$. For example, it avoids inclusion of two triplets $(i,j,j_1)$ and $(i,j_2, j)$ in cover where  $(i,j_1)$ and $(i,j_2)$ both are in $Locked(\theta)$. This is not realizable because we can not chose a $\tilde{\mu}^C_{i,j}$ that simultaneously bigger than $\mu_{i,j_2}(\theta)$ and smaller than $\mu_{i,j_1}(\theta)$, where $\mu_{i,j_1}(\theta) < \mu_{i,j_2}(\theta)$. Condition (iii) is present because we need to respect the $\theta$-locked property of any $\lambda\in\Lambda_{\mathcal A}(\theta)$.

\section{Proof of Results in Section~\ref{app:lower}}\label{sec:proof_lower_bound}
\begin{lemma}[Domination of admissible-rank alternatives]\label{lemm:alternative_domination}
For $\theta\in\Theta$,
$\Lambda(\theta)\subseteq\Lambda_{\mathcal A}(\theta)$.
For each $\lambda\in\Lambda_{\mathcal A}(\theta)$, set
$\bar\lambda_e=\max\{\theta_e,\lambda_e\}$ at every reward coordinate $e$.
There is a sequence $\nu_n\in\Lambda(\theta)$ converging to $\bar\lambda$ such that, for every matching $M$,
\[
\lim_{n\to\infty}kl(\theta,\nu_n;M)
=kl(\theta,\bar\lambda;M)
\leq kl(\theta,\lambda;M).
\]
Consequently the two alternative families give identical information constraints on nonnegative matching allocations.
\end{lemma}
\begin{proof}
If $F(\lambda)$ extends an admissible rank $P$ for $\theta$, any $M\in\mathrm{SuperStable}(P)$ is stable at both $\theta$ and $\lambda$. This proves the inclusion.

Fix $\lambda\in\Lambda_{\mathcal A}(\theta)$ and $M\in\mathrm{Stable}(F(\theta))$. Suppose every pair $(i,j)\notin M$ has an endpoint that strictly prefers its partner in $M$ to that pair under $\bar\lambda$. Since each matched partner edge of $M$ is locked, with the unmatched option having fixed reward zero, its two means are equal at $\theta$ and $\lambda$, whereas every outside coordinate under $\bar\lambda$ is the maximum of its two original means. Each selected strict protective comparison therefore holds at both $\theta$ and $\lambda$. Select one for each outside pair. These comparisons form acyclic partial orders: at each endpoint every selected comparison points from its fixed $M$-partner to another partner. They define a partial rank $P_M$ compatible with both full profiles, and $M$ is super-stable under $P_M$. Hence $P_M\in\mathcal A(\theta)$ and $F(\lambda)\in \bigcup_{P\in\mathcal A(\theta)}\mathrm{FullRank}(P)$, a contradiction.

Thus, for every $M\in\mathrm{Stable}(F(\theta))$, some pair $(i,j)\notin M$ has neither endpoint strictly protecting $M$ under $\bar\lambda$. The pair $(i,j)$ is unlocked. Indeed, if it were locked, both its means under $\bar\lambda$ would equal their $\theta$ values; at least one endpoint would then strictly protect $M$, since $M$ is stable at $\theta$. The matched partner means are locked as well. Both outside means of this witness are therefore at least their respective partner means.

There are finitely many true stable matchings. For each, select such an unlocked witness and increase its two outside means slightly; increasing a coordinate selected for several witnesses only reinforces their blocking inequalities. Perturb remaining unlocked coordinates by still smaller generic amounts to remove ties, leaving all locked means fixed. Because $\theta$ and $\lambda$ are interior and $\bar\lambda$ is their coordinatewise maximum, a sequence of perturbations tending to zero remains in $(0,1)^{2NK}$. Every selected pair strictly blocks its corresponding true stable matching. The resulting strict parameters $\nu_n$ therefore belong to $\Lambda(\theta)$ and converge to $\bar\lambda$.

At each coordinate, $\bar\lambda_e$ equals either $\theta_e$ or $\lambda_e$. Its Bernoulli KL divergence from $\theta_e$ is accordingly zero or exactly the contribution from $\lambda_e$. KL is continuous at these interior coordinates, giving the asserted limits simultaneously for the finite action set. If a nonnegative allocation satisfies the information constraint for each $\nu_n$, passage to the limit shows that it satisfies the constraint for $\lambda$. Together with the inclusion already proved, this yields equality of the two feasible sets.
\end{proof}

\begin{proof}[Proof of Theorems~\ref{thm:centralized_lower} and~\ref{thm:centralized_lower_combes}]
Fix $\nu\in\Lambda(\theta)$ and write $N_M(T)$ for the number of plays of $M$. Let $E_T$ be the event that at least $T/2$ plays are from $\mathrm{Stable}(\theta)$. Uniform goodness implies $\Pr_\theta(E_T)\to1$. Under $\nu$, each of those plays is unstable, so for every $a>0$, Markov's inequality gives $\Pr_\nu(E_T)=O(T^{a-1})$. Data processing and the adaptive likelihood-ratio chain rule give
\[
\sum_M\mathbb E_\theta N_M(T)\,kl(\theta,\nu;M)
\geq kl\bigl(\Pr_\theta(E_T),\Pr_\nu(E_T)\bigr)
\geq(1-a-o(1))\log T.
\]
Let $a\downarrow0$. Stable actions contribute zero KL because every edge they use is locked. If the liminf is infinite there is nothing to prove. Otherwise, along a subsequence attaining the finite liminf of expected binary regret divided by $\log T$, finite dimensional compactness yields a limiting allocation on unstable matchings satisfying all the constraints for each fixed $\nu\in\Lambda(\theta)$. Its objective is the liminf. Finally $\Lambda(\theta)\subseteq\Lambda_{\mathcal A}(\theta)$, and Lemma~\ref{lemm:alternative_domination} shows that each extra admissible-rank constraint is implied by a true alternative constraint. Hence the two feasible sets and their LP values agree.
\end{proof}

Before proceeding to the proof, we need to define the set $kl_{>0}(\lambda) = \{(i,j): kl(\theta, \lambda; (i,j)) \neq 0\}$ contain the pairs where $\lambda$ changes from $\theta$, and for any set of instances $A$, $kl_{>0}(A) = \cup_{\lambda \in A} kl_{>0}(\lambda)$.

We also require the following two lemmas that help with deriving bounds for finite-support dual bound~\eqref{eq:dual}.

The following lemma shows how $\Lambda(\theta)$ can be restricted to mutually exclusive observation supports and can be used in bounding the finite-support dual bound~\eqref{eq:dual}.

\begin{lemma}\label{lemm:dual_lower}
Consider nonempty subsets $\Lambda_l\subseteq\Lambda_{\mathcal A}(\theta)$, $l\in[L]$, such that every matching intersects $kl_{>0}(\Lambda_l)$ for at most one $l$. If
$0<B_l:=\sup_{\lambda\in\Lambda_l}\max_{M\notin\mathrm{Stable}(\theta)}kl(\theta,\lambda;M)<\infty$, then $c(\theta)\geq\sum_{l=1}^L B_l^{-1}$.
\end{lemma}

Next, we provide a kl divergence bound for instances under the set $\Lambda_C$ for cover $C$, as defined in \ref{def:cover}.
\begin{lemma}\label{lemm:opt_lambda_c}
For every feasible cover $C$ and every $\delta>0$, there is $\lambda_{C,\delta}\in\Lambda_C$ such that
$\max_{M\notin\mathrm{Stable}(\theta)}kl(\theta,\lambda_{C,\delta};M)\leq kl(\theta,C)+\delta$ and $kl_{>0}(\lambda_{C,\delta})\subseteq D_C$.
\end{lemma}

\begin{proof}[Proof of Theorem~\ref{thm:centralized_lower_final}]
(Part 1) A finite-support nonnegative allocation $\iota$ gives a valid dual lower bound whenever
\begin{align}\label{eq:dual}
\sum_{\lambda\in\Lambda_{\mathcal A}(\theta)}\iota(\lambda)kl(\theta,\lambda;M)\leq1
\quad\text{for every }M\notin\mathrm{Stable}(\theta).
\end{align}
Indeed, multiply each information constraint in Equation~\eqref{eq:primal} by $\iota(\lambda)$ and sum. The finite support justifies interchanging sums, and the displayed inequalities imply $c(\theta)\geq\sum_\lambda\iota(\lambda)$. No strong-duality assertion for an infinite-dimensional program is needed.

For every $\lambda\in\Lambda_C$, Equation~\eqref{eq:lower_4} preserves all locked means. Equations~\eqref{eq:lower_5}--\eqref{eq:lower_6} reverse a comparison in every boundary admissible rank. Lemma~\ref{lemm:lambda_decomp} therefore gives $\lambda\in\Lambda_{\mathcal A}(\theta)$. Thus restricting the finite-support dual allocation to $\bigcup_C\Lambda_C$ yields the valid lower bound in Equation~\eqref{eq:lower_1}. Lemma~\ref{lemm:alternative_domination} is what permits these admissible-rank constraints even when an alternative still shares a stable matching with $\theta$.

(Part 2) Fix $G\in\mathcal{CG}(\theta)$ and $\delta>0$. For each $C\in G$, select $\lambda_{C,\delta}$ from Lemma~\ref{lemm:opt_lambda_c} and assign it dual mass $1/(kl(\theta,C)+\delta)$. All other masses are zero. Each matching intersects the changed support of at most one selected alternative, by Definition~\ref{def:cover_group}. Its total weighted KL contribution is therefore at most one. Weak duality gives
\[
c(\theta)\geq\sum_{C\in G}\frac{1}{kl(\theta,C)+\delta}.
\]
Let $\delta\downarrow0$ and maximize over the finite family of exclusive cover groups. This proves the second bound.
\end{proof}

\subsection{Additional Proofs required for Proof of Theorem~\ref{thm:centralized_lower_final}} \label{app:additional_proofs}
\begin{proof}[Proof of Lemma~\ref{lemm:lambda_decomp}]
For any $\lambda\in\Lambda_{\mathcal A}(\theta)$, all means on locked pairs equal those of $\theta$ by definition. Thus condition (i) holds. We next assume that the full rank $(F_u(\lambda), F_a(\lambda))$ reverses no $\theta$-open triplet for some partial rank $(\underline{P_u}(\theta), \underline{P_a}(\theta)) \in \mathcal{B}(\theta)$. Consider any $i,i'\in [N]$ and any $j,j'\in [K]$. Then $j \underset{F_{u,i}(\lambda)}{>} j'$ implies $j' \underset{\underline{P_{u,i}}(\theta)}{\not>} j$. Similarly, $i \underset{F_{a, j}(\lambda)}{>} i'$ implies $i' \underset{\underline{P_{a,j}}(\theta)}{\not>} i$. Hence, $(F_u(\lambda), F_a(\lambda))$ is compatible with $(\underline{P_u}(\theta), \underline{P_a}(\theta))$, and $(F_u(\lambda), F_a(\lambda)) \in FullRank(\underline{P_u}(\theta), \underline{P_a}(\theta))$. But that leads to a contradiction as $(F_u(\lambda), F_a(\lambda)) \notin \cup_{(P_u, P_a)\in \mathcal{A}(\theta)}  FullRank(P_u, P_a)$. So, $(F_u(\lambda), F_a(\lambda))$ must reverse at least one $\theta$-open triplet for each $(P_u, P_a) \in \mathcal{B}(\theta)$. This implies that for all $\lambda \in \Lambda_{\mathcal A}(\theta)$ the conditions $(i)$ and $(ii)$ in the lemma statement are satisfied. 

To prove the other direction, for some $\lambda$ let conditions $(i)$ and $(ii)$ are satisfied. Then due to $(i)$, $\lambda$ satisfies $\mu_{ij}(\theta) = \mu_{ij}(\lambda), \gamma_{ji}(\theta) = \gamma_{ji}(\lambda),\forall (i,j) \in \mathrm{Locked}(\theta)$. Next due to $(ii)$, $(F_u(\lambda), F_a(\lambda))$ reverses at least one $\theta$-open triplet for each $(P_u, P_a) \in \mathcal{B}(\theta)$. But that means  $(F_u(\lambda), F_a(\lambda))$ is no longer compatible with any $(P_u, P_a) \in \mathcal{B}(\theta)$. This in turn implies, because every admissible rank refines an inclusion-minimal admissible rank that $(F_u(\lambda), F_a(\lambda))$ is not compatible with any $(P_u, P_a) \in \mathcal{A}(\theta)$. This completes the proof.
\end{proof}

\begin{proof}[Proof of Corollary~\ref{corr:wide_sub_log}]
If a boundary admissible rank has no $\theta$-open comparison, each of its comparisons is between locked coordinates. An alternative preserving locked means cannot reverse any of them. By Lemma~\ref{lemm:lambda_decomp}, $\Lambda_{\mathcal A}(\theta)=\emptyset$, and hence $\Lambda(\theta)=\emptyset$. The information program then has the zero allocation as a feasible solution and $c(\theta)=0$.

For achievability at a fixed target $\theta$, choose such a partial rank $P^0$ and $M^0\in\mathrm{SuperStable}(P^0)$. Fix a finite list $\mathcal S$ of $\theta$-stable matchings covering every locked pair needed by the comparisons in $P^0$. These choices are constants of a horizon-dependent policy, which is then evaluated at every possible true parameter $\lambda\in\Theta$; the policy does not assume the target is correct. This auxiliary existence construction is distinct from the proposed anytime algorithms. At horizon $T$, play each matching in $\mathcal S$ exactly $\lceil\log^2(2T)\rceil$ times, or stop if the horizon is exhausted. Form confidence intervals with radius $\sqrt{6\log(2T)/n}$. If every comparison of $P^0$ is certified, play $M^0$ for the remaining rounds. Otherwise, run the centralized algorithm with fresh samples for the remainder, using its own elapsed-round counter in the anytime confidence radii.

For the validation step, let $\mathcal H_T(\lambda)$ be the event that, for every endpoint reward stream and each deterministic sample count $1\leq s\leq T$, its empirical mean is within $\sqrt{6\log(2T)/s}$ of its mean at $\lambda$. A union bound over $2NK$ streams and $T$ sample counts gives $\Pr_\lambda(\mathcal H_T(\lambda)^c)\leq 4NKT\exp(-3\log(2T))=NK/(2T^2)$.
At any $\lambda\in\Theta$, the initial stage costs at most $|\mathcal S|\lceil\log^2(2T)\rceil$ binary regret. On $\mathcal H_T(\lambda)$, a certified $P^0$ is compatible with $F(\lambda)$, so committing to $M^0$ costs no regret. The probability of an incorrect certification contributes at most $NK/(2T)$, and the fallback has $O_\lambda(\log T)$ regret by Theorem~\ref{thm:centralized_main}, also conditional on the validation data, since it uses fresh independent observations. Thus this fixed policy is uniformly good on $\Theta$.

At the target $\theta$, all initial plays are stable. Every comparison in $P^0$ has a fixed positive gap and both of its coordinates receive at least $\lceil\log^2(2T)\rceil$ samples. For all sufficiently large $T$, their confidence intervals certify every comparison on $\mathcal H_T(\theta)$. Hence only its complement can cause further regret, contributing at most $NK/(2T)$. The finitely many smaller horizons are absorbed into the instance-dependent constant, proving $O_\theta(1)$ regret. If $P^0$ has no comparisons, its super-stable matching can be played directly.
\end{proof}

\begin{proof}[Proof of Lemma~\ref{lemm:dual_lower}]
Select one $\lambda_l\in\Lambda_l$ for each $l$ and assign it dual mass $1/B_l$. A matching encounters at most one changed support, so $\sum_{l=1}^L kl(\theta,\lambda_l;M)/B_l\leq1$. This is a feasible finite-support dual allocation, proving the bound. The hypothesis is exclusivity within a matching, not merely disjointness of changed-edge sets.
\end{proof}

\begin{proof}[Proof of Lemma~\ref{lemm:opt_lambda_c}]
The defining infimum in Equation~\eqref{eq:kl_theta_c} is finite, and $\Lambda_C$ is nonempty. Hence for every $\delta>0$ there is a feasible $\lambda_{C,\delta}$ with worst-matching divergence at most $kl(\theta,C)+\delta$. The support inclusion follows from the requirement in the definition of $\Lambda_C$ that all reward coordinates outside $D_C$ remain unchanged.
\end{proof}

\end{document}